\documentclass{article}

\usepackage{microtype}
\usepackage{graphicx}
\usepackage{subfigure}
\usepackage{booktabs} 

\usepackage{hyperref}

\usepackage[accepted]{icml2019}

\usepackage[utf8]{inputenc} 
\usepackage[T1]{fontenc}    
\usepackage{url}            
\usepackage{amsfonts}       
\usepackage{nicefrac}       
\usepackage{amsmath}
\usepackage{amssymb}
\usepackage{siunitx}
\usepackage{bbm}
\usepackage{multirow}
\usepackage{pifont}
\usepackage{hhline}
\usepackage{xcolor,colortbl}
\definecolor{Gray}{gray}{0.85}

\newif\ifpaper
 \paperfalse

\def\RR{\mathbb{R}}
\def\>{\rangle}
\def\rank{\operatorname{\textit{rank}}}

\def\Set#1{\left\{ #1 \right\}}
\def\Bigbar#1{\mathrel{\left|\vphantom{#1}\right.}}
\def\Setbar#1#2{\Set{#1 \Bigbar{#1 #2} #2}}

\newtheorem{theorem}{Theorem}[section]
\newtheorem{lemma}[theorem]{Lemma}
\newtheorem{definition}[theorem]{Definition}
\newtheorem{proposition}[theorem]{Proposition}

\newtheorem{assumptions}[theorem]{Assumption}
\newtheorem{property}[theorem]{Property}
\newenvironment{proof}{\par\noindent{\bf Proof:\ }}{\hfill$\Box$\\[2mm]}
\newenvironment{proof1}{\par\noindent{}}{\hfill$\Box$\\[2mm]}
\newcommand{\Id}{\mathbb{I}}

\def\Span{\textrm{Span}}
\newcommand{\norm}[1]{\left\|#1\right\|}

\def\bydef{\mathrel{\mathop:}=}

\def\range{\mathop{\rm range}\nolimits}

\def\ker{\mathop{\rm ker}\nolimits}

\def\dim{\mathop{\rm dim}\nolimits}

\def\min{\mathop{\rm min}\nolimits}
\def\max{\mathop{\rm max}\nolimits}
\def\ones{\mathbf{1}}

\def\ie{\textit{i.e. }}

\icmltitlerunning{On Connected Sublevel Sets in Deep Learning}

\begin{document}

\twocolumn[
\icmltitle{On Connected Sublevel Sets in Deep Learning}



\icmlsetsymbol{equal}{*}

\begin{icmlauthorlist}
\icmlauthor{Quynh Nguyen}{saar}
\end{icmlauthorlist}

\icmlaffiliation{saar}{Department of Mathematics and Computer Science, Saarland University, Germany}

\icmlcorrespondingauthor{Quynh Nguyen}{quynh@cs.uni-saarland.de}

\icmlkeywords{loss landscape, deep neural networks, local minima}

\vskip 0.3in
]



\printAffiliationsAndNotice{}  

\begin{abstract}

This paper shows that every sublevel set of the loss function of a class of deep over-parameterized neural nets with piecewise linear activation functions
is connected and unbounded.
This implies that the loss has no bad local valleys and all of its global minima are connected
within a unique and potentially very large global valley.

\end{abstract}

\section{Introduction}
It has been commonly observed in deep learning that over-parameterization can be helpful for optimizing deep neural networks.
In particular, several recent work \cite{AllenZhuEtal2018,DuEtAl2018_GD,ZouEtal2018} 
have shown that if ``all the hidden layers'' of a deep network have polynomially large number of neurons 
compared to the number of training samples and the network depth,
then (stochastic) gradient descent converges to a global minimum with zero training error.
While these theoretical guarantees are interesting conceptually, 
it remains largely unclear why this kind of simple local search algorithms 
can succeed given the well-known non-convexity and NP-Hardness of the problem.
We are interested in the following questions: 

\textit{Why local search algorithms such as (stochastic) gradient descent do not 
seem to get stuck at bad valleys under excessive over-parameterization regimes? 
Is there any geometric structure of the loss function that can ``intuitively'' support
for the successes of these algorithms under such regimes?}

In this paper, we shed light on these questions by showing that 
every sublevel set of the loss is connected if ``one of the hidden layers'' is wide enough.
While connectivity of sublevel sets does not ensure that gradient descent always converges to a global minimum from arbitrary initialization,
such simple geometric structure still arguably makes the loss function much more favorable to local search algorithms 
than any other ``wildly'' non-convex functions.
An example function which satisfies this property is shown in Figure \ref{fig:connected_sublevel_set}.
\begin{figure}
\vspace{-10pt}
\begin{center}
\includegraphics[width=0.7\columnwidth]{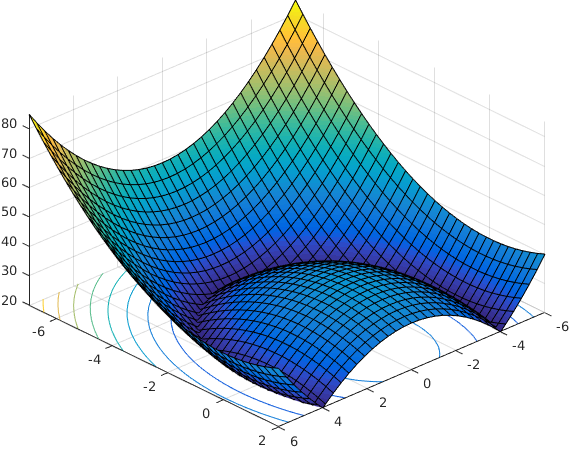}
\end{center}
\vspace{-8pt}
\caption{A non-convex function with connected sublevel sets.}
\label{fig:connected_sublevel_set}
\vspace{-15pt}
\end{figure}

The key idea of our paper is to first prove the connectivity of sublevel sets for arbitrary sized neural networks 
in the regime where the training data are linearly independent.
Then we extend such result to ``arbitrary training data'' by assuming that the network has a wide hidden layer.
More specifically, we show that if one of the hidden layers has more neurons than the number of training samples,
then the loss function has no bad local valleys in the sense that 
there is a continuous path from any starting point in parameter space on which the loss is non-increasing and 
gets arbitrarily close to its (asymptotic) minimal value.
For a special case of the network where the first hidden layer has twice more neurons than the number of training samples,
we show that every sublevel set of the loss becomes connected. 
This is a stronger property than before as it not only implies that 
the loss has no bad local valleys,
but also that all finite global minima (if exist) are connected within a unique valley.
All our results hold for standard deep fully connected networks with arbitrary convex losses and piecewise linear activation functions.
All missing proofs can be found in the appendix.

\begin{figure*}
\begin{center}
\includegraphics[height=0.45\columnwidth]{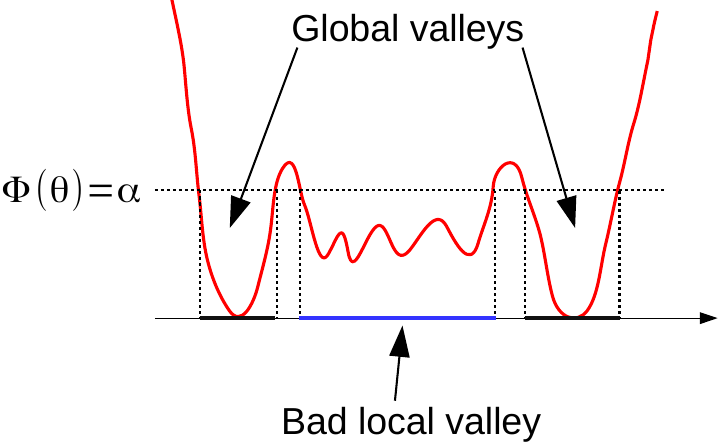}\hspace{10pt}
\includegraphics[height=0.45\columnwidth]{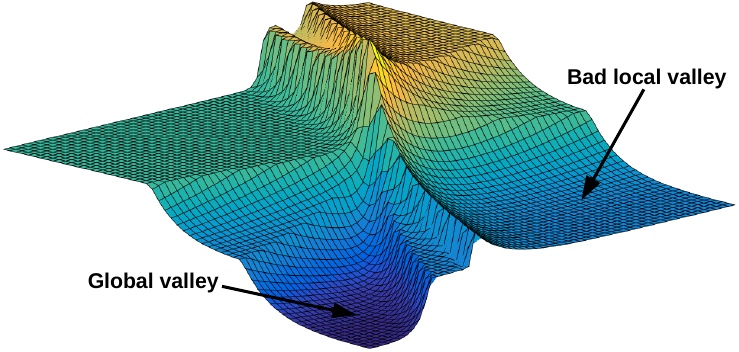}
\end{center}
\caption{Examples of bad local valleys in one and two dimension. 
(Recall from Definition \ref{def:bad_valley}: A bad local valley is a connected component of some strict sublevel set
on which the loss cannot be made arbitrarily close to the infimum. 
Intuitively, bad local valleys are regions in the search space where gradient descent can easily get stuck.)}
\label{fig:bad_valleys}
\end{figure*}
\section{Background}
Let $N$ be the number of training samples and $X=[x_1,\ldots,x_N]^T\in\RR^{N\times d}$ the training data with $x_i\in\RR^d.$
Let $L$ be the number of layers of the network, 
$n_k$ the number of neurons at layer $k,$
$d$ the input dimension, $m$ the output dimension,
and $W_k\in\RR^{n_{k-1}\times n_k}$ and $b_k\in\RR^{n_k}$ the weight matrix and biases respectively of layer $k.$
By convention, we assume that $n_0=d$ and $n_L=m.$
Let $\sigma:\RR\to\RR$ be a continuous activation function specified later.
The network output at layer $k$ is the matrix $F_k\in\RR^{N\times n_k}$ defined as
\begin{align}\label{eq:F_k}
    F_k = 
    \begin{cases}
	X & k=0\\
	\sigma \big( F_{k-1} W_k + \ones_N b_k^T \big) & 1\leq k\leq L-1\\
	F_{L-1} W_L + \ones_N b_L^T & k=L
    \end{cases}
\end{align}
Let $\theta\bydef (W_l,b_l)_{l=1}^L$ be the set of all parameters.
Let $\Omega_l$ be the space of $(W_l,b_l)$ for every layer $l\in[1,L],$
and $\Omega=\Omega_1\times\ldots\times\Omega_L$ the whole parameter space.
Let $\Omega_l^*\subset\Omega_l$ be the subset of parameters of layer $l$ for which 
the corresponding weight matrix has full rank, that is $\Omega_l^*=\Setbar{(W_l,b_l)}{W_l \textrm{ has full rank}}.$
In this paper, we often write $F_k(\theta)$ to denote the network output at layer $k$ as a function $\theta$,
but sometimes we drop the argument $\theta$ and write just $F_k$ if it is clear from the context.
We also use the notations $F_k\Big((W_1,b_1),\ldots,(W_L,b_L)\Big),F_k\Big((W_1,b_1),(W_l,b_l)_{l=2}^L\Big).$
The training loss $\Phi:\Omega\to\RR$ is defined as
\begin{align}\label{eq:train_loss}
    \Phi(\theta) = \varphi(F_L(\theta))
\end{align}
where $\varphi:\RR^{N\times m}\to\RR$ is assumed to be any convex loss.
Typical examples include the standard cross-entropy loss
$\varphi(F_L) = \frac{1}{N}\sum_{i=1}^N - \log \Big(\frac{e^{(F_L)_{iy_i}}}{\sum_{k=1}^m e^{(F_L)_{ik}}}\Big)$
where $y_i$ is the ground-truth class of $x_i,$
and the standard square loss for regression
$\varphi(F_L)=\frac{1}{2}\norm{F_L-Y}_F^2$ 
where $Y\in\RR^{N\times m}$ is a given training output.

In this paper, we denote $p^*=\inf_{G\in\RR^{N\times m}}\varphi(G)$ which serves as a lower bound on $\Phi$.
Note that $p^*$ is fully determined by the choice of $\varphi(\cdot)$ and thus is independent of the training data.
Please also note that we make no assumption about finiteness of $p^*$ in this paper
although for most of practical loss functions as mentioned above one has $p^*=0.$
We list below several assumptions on the activation function
and will refer to them accordingly in our different results.
\begin{assumptions}\label{ass:act1}
    $\sigma$ is strictly monotonic and $\sigma(\RR)=\RR.$
\end{assumptions}
Note that Assumption \ref{ass:act1} implies that $\sigma$ has a continuous inverse $\sigma^{-1}:\RR\to\RR,$
which is satisfied for Leaky-ReLU.
\begin{assumptions}\label{ass:act2}
    There do not exist non-zero coefficients $(\lambda_i,a_i)_{i=1}^p$ with $a_i\neq a_j\,\forall\,i\neq j$ 
    such that $\sigma(x)=\sum_{i=1}^p\lambda_i\sigma(x-a_i)$ for every $x\in\RR.$
\end{assumptions}
Assumption \ref{ass:act2} is satisfied for every piecewise linear activation functions except the linear one as shown below.
\begin{lemma}\label{lem:exp_act}
    Assumption \ref{ass:act2} is satisfied for any continuous piecewise linear activation function 
    with at least two pieces such as ReLU, Leaky-ReLU, etc
    and for the exponential linear unit 
    $\sigma(x)=\begin{cases}x&x\ge0\\ \alpha(e^x-1)&x<0\end{cases} \textrm{ where }\alpha>0.$
\end{lemma}
Through out the rest of this paper, we will make the following mild assumption on our training data.
\begin{assumptions}\label{ass:data}
    All the training samples are distinct.
\end{assumptions}
A key concept of this paper is the sublevel set of a function.
\begin{definition}\label{def:sublevel_set}
    For every $\alpha\in\RR$, 
    the $\alpha$-level set of $\Phi:\Omega\to\RR$ is the preimage
    $\Phi^{-1}(\alpha)=\Setbar{\theta\in\Omega}{\Phi(\theta)=\alpha},$
    and the $\alpha$-sublevel set of $\Phi$ is given as
    $L_{\alpha}=\Setbar{\theta\in\Omega}{\Phi(\theta) \leq \alpha}.$
\end{definition}
Below we recall the standard definition of connected sets and some basic properties which are used in this paper.
\begin{definition}[Connected set]\label{def:connected_set}
    A subset $S\subseteq\RR^d$ is called connected if for every $x,y\in S$, 
    there exists a continuous curve $r: [0,1]\to S$ such that $r(0)=x$ and $r(1)=y.$
\end{definition}
\begin{proposition}\label{prop:connected_continuous_map}
    Let $f:U\to V$ be a continuous map.
    If $A\subseteq U$ is connected then $f(A)\subseteq V$ is also connected.
\end{proposition}
\begin{proposition}\label{prop:minkowski}
    The Minkowski sum of two connected subsets $U,V\subseteq\RR^n$,
    defined as $U+V=\Setbar{u+v}{u\in U,v\in V}$, is a connected set.
\end{proposition}
In this paper, $A^\dagger$ denotes the Moore-Penrose inverse of $A.$
If $A$ has full row rank then it has a right inverse $A^\dagger=A^T(AA^T)^{-1}$ with $AA^\dagger=\Id$,
and if $A$ has full column rank then it has a left inverse $A^\dagger=(A^TA)^{-1}A^T$ with $A^\dagger A=\Id.$
\section{Key Result: Linearly Independent Data Leads to Connected Sublevel Sets}
This section presents our key results for linearly independent data,
which form the basis for our additional results in the next sections where we analyze deep over-parameterized networks 
with arbitrary data.
Below we assume that the widths of all hidden layers are decreasing, i.e. $n_1>\ldots> n_L.$
Note that it is still possible to have $n_1\geq d$ or $n_1<d.$
The above condition is quite natural
as in practice (e.g., see Table 1 in \cite{Quynh2018b})
the first hidden layer often has the most number of neurons,
afterwards the number of neurons starts decreasing towards the output layer, 
which is helpful for the network to learn more compact representations at higher layers.
We introduce the following key property for a class of points $\theta=(W_l,b_l)_{l=1}^L$ in parameter space and refer to it later
in our theorems and proofs.
\begin{property}\label{P1}
    $W_l$ has full rank for every $l\in[2,L].$
\end{property}
Our main result in this section is stated as follows.
\begin{theorem}\label{thm:lin_data}
    Let Assumption \ref{ass:act1} hold, $\rank(X)=N$ and $n_1> \ldots> n_L$ where $L\geq 2.$
    Then the following hold:
    \begin{enumerate}
	\item Every sublevel set of $\Phi$ is connected.
	Moreover, $\Phi$ can attain any value arbitrarily close to $p^*.$
	\item Every non-empty connected component of every level set of $\Phi$ is unbounded.
    \end{enumerate}
\end{theorem}
We have the following decomposition of sublevel set:
$\Phi^{-1}((-\infty,\alpha]) = \Phi^{-1}(\alpha)\cup\Phi^{-1}((-\infty,\alpha)).$
It follows that if $\Phi$ has unbounded level sets then its sublevel sets must also be unbounded.
We note that the reverse is not true, e.g. the standard Gaussian distribution function has unbounded sublevel sets
but its level sets are bounded.
Given that, the two statements of Theorem \ref{thm:lin_data} together imply that 
every sublevel set of the loss must be both connected and unbounded.
While the connectedness property of sublevel sets implies that the loss function is rather well-behaved,
the unboundedness property of level sets intuitively implies that 
there are no bounded valleys in the loss surface,
regardless of whether these valleys contain a global minimum or not.
Clearly this also indicates that $\Phi$ has no strict local minima/maxima.
In the remainder of this section, we will present the proof of Theorem \ref{thm:lin_data}.
The following lemmas will be helpful.
\begin{lemma}\label{lem:canonical_form}
    Let the conditions of Theorem \ref{thm:lin_data} hold.
    Given some $k\in[2,L].$ 
    Then there is a continuous map $h:\Omega_2^*\times\ldots\times\Omega_k^*\times\RR^{N\times n_k}\to\Omega_1$
    which satisfy the following:
    \begin{enumerate}
	\item For every $\Big((W_2,b_2),\ldots,(W_k,b_k),A\Big)\in\Omega_2^*\times\ldots\times\Omega_k^*\times\RR^{N\times n_k}$
	it holds that $F_k\Big(h\Big((W_l,b_l)_{l=2}^k,A\Big),(W_l,b_l)_{l=2}^k\Big)=A.$
	\item For every $\theta=(W_l^*,b_l^*)_{l=1}^L$ where all the matrices $(W_l^*)_{l=2}^k$ have full rank,
	      there is a continuous curve from $\theta$ to 
	      $\Big(h\Big((W_l^*,b_l^*)_{l=2}^k, F_k(\theta)\Big),(W_l^*,b_l^*)_{l=2}^L\Big)$ 
	      on which the loss $\Phi$ is constant.
    \end{enumerate}
\end{lemma}
\begin{proof}
    For every $\Big((W_2,b_2),\ldots,(W_k,b_k),A\Big)\in\Omega_2^*\times\ldots\times\Omega_k^*\times\RR^{N\times n_k},$
    let us define the value of the map $h$ as
    \begin{align*}
	h\Big((W_l,b_l)_{l=2}^k,A\Big) = (W_1,b_1), 
    \end{align*}
    where $(W_1,b_1)$ is given by the following recursive formula
    \begin{align*}
	\begin{cases}
	    \begin{bmatrix}W_1\\b_1^T\end{bmatrix}=[X,\ones_N]^\dagger\sigma^{-1}(B_1),\\
	    B_l = \Big(\sigma^{-1}(B_{l+1})-\ones_Nb_{l+1}^T\Big)\, W_{l+1}^\dagger,\,\forall\,l\in[1,k-2],\\
	    B_{k-1} = \begin{cases}
			  (A-\ones_Nb_L^T)\,W_L^\dagger & k=L \\
			  \Big(\sigma^{-1}(A)-\ones_Nb_{k}^T\Big)\, W_k^\dagger & k\in[2,L-1] 
		      \end{cases}
	\end{cases}.
    \end{align*}
    By our assumption $n_1>\ldots>n_L,$
    it follows from the domain of $h$ that
    all the matrices $(W_l)_{l=2}^k$ have full column rank, and so they have a left inverse.
    Similarly, $[X,\ones_N]$ has full row rank due to our assumption that $\rank(X)=N$, and so it has a right inverse. 
    Moreover $\sigma$ has a continuous inverse by Assumption \ref{ass:act1}.
    Thus $h$ is a continuous map as it is a composition of continuous functions.
    In the following, we prove that $h$ satisfies the two statements of the lemma.
    
    \textit{1.} 
    Let $\Big((W_2,b_2),\ldots,(W_k,b_k),A\Big)\in\Omega_2^*\times\ldots\times\Omega_k^*\times\RR^{N\times n_k}.$
    Since all the matrices $(W_l)_{l=2}^k$ have full column rank and $[X,\ones_N]$ has full row rank,
    it holds that $W_l^\dagger W_l=\Id$ and $[X,\ones_N][X,\ones_N]^\dagger=\Id$
    and thus we easily obtain from the above definition of $h$ that
    \begin{align*}
	\begin{cases}
	    B_1 = \sigma\Big([X,\ones_N]\begin{bmatrix}W_1\\b_1^T\end{bmatrix}\Big),\\
	    B_{l+1} = \sigma(B_l W_{l+1}+\ones_Nb_{l+1}^T),\,\forall\,l\in[1,k-2],\\
	    A = \begin{cases}B_{L-1}W_L+\ones_Nb_L^T & k=L, \\ \sigma(B_{k-1}W_k+\ones_Nb_k^T) & k\in[2,L-1].\end{cases}
	\end{cases}
    \end{align*}
    One can easily check that the above formula of $A$ is exactly the definition of $F_k$ from \eqref{eq:F_k}
    and thus it holds $F_k\Big(h\Big((W_l,b_l)_{l=2}^k,A\Big),(W_l,b_l)_{l=2}^k\Big)=A$
    for every $\Big((W_2,b_2),\ldots,(W_k,b_k),A\Big)\in\Omega_2^*\times\ldots\times\Omega_k^*\times\RR^{N\times n_k}.$ 
    
    \textit{2.}
    Let $G_l:\RR^{N\times n_{l-1}}\to\RR^{N\times n_l}$ be defined as
    \begin{align*}
	G_l(Z) = \begin{cases}ZW_L^*+\ones_N(b_L^*)^T & l=L\\ \sigma\Big(ZW_l^*+\ones_N(b_l^*)^T\Big) & l\in[2,L-1].\end{cases}
    \end{align*}
    For convenience, let us group the parameters of the first layer into a matrix, 
    say $U=[W_1^T,b_1]^T\in\RR^{(d+1)\times n_1}.$
    Similarly, let $U^*=[(W_1^*)^T,b_1^*]^T\in\RR^{(d+1)\times n_1}.$
    Let $f:\RR^{(d+1)\times n_1}\to\RR^{N\times n_k}$ be a function of $(W_1,b_1)$ defined as
    \begin{align*}
	&f(U) = G_k\circ G_{k-1}\ldots G_2 \circ G_1(U), \textrm{ where} \\
	&G_1(U)=\sigma([X,\ones_N]U),\; U=[W_1^T,b_1]^T.
    \end{align*}
    We note that this definition of $f$ is exactly $F_k$ from \eqref{eq:F_k},
    but here we want to exploit the fact that $f$ is a function of $(W_1,b_1)$
    as all other parameters are fixed to the corresponding values of $\theta.$
    Let $A=F_k(\theta).$
    By definition we have $f(U^*)=A$ and thus $U^*\in f^{-1}(A).$
    Let us denote 
    \begin{align*}
	(W_1^h,b_1^h)=h\Big(W_l^*,b_l^*)_{l=2}^k,A\Big),\quad U^h=[(W_1^h)^T,b_1^h]^T.
    \end{align*}
    By applying the first statement of the lemma to $\Big((W_2^*,b_2^*),\ldots,(W_k^*,b_k^*),A\Big)$
    we have
    \begin{align*}
	A = F_k\Big((W_1^h,b_1^h),(W_l^*,b_l^*)_{l=2}^k\Big)= f(U^h)
    \end{align*}
    which implies $U^h\in f^{-1}(A).$
    So far both $U^*$ and $U^h$ belong to $f^{-1}(A).$
    The idea now is that if one can show that $f^{-1}(A)$ is a connected set then there would exist a connected path between $U^*$ and $U^h$ 
    (and thus a path between $(W_1^*,b_1^*)$ and $(W_1^h,b_1^h)$)
    on which the output at layer $k$ is identical to $A$ and hence the loss is invariant, 
    which concludes the proof.
    
    In the following, we show that $f^{-1}(A)$ is indeed connected.
    First, one observes that $\range(G_l)=\RR^{N\times n_l}$ for every $l\in[2,k]$ since
    all the matrices $(W_l^*)_{l=2}^k$ have full column rank
    and $\sigma(\RR)=\RR$ due to Assumption \ref{ass:act1}.
    Similarly, it follows from our assumption $\rank(X)=N$ that $\range(G_1)=\RR^{N\times n_1}.$
    By standard rules of compositions, we have
    \begin{align*}
	f^{-1}(A) 
	&=G_1^{-1}\circ G_2^{-1}\circ\ldots\circ G_k^{-1} (A).
    \end{align*}
    where all the inverse maps $G_l^{-1}$ have full domain.
    It holds
    \begin{align*}
	&G_k^{-1}(A)=\\
	&\begin{cases}(A-\ones_Nb_L^T)(W_L^*)^\dagger+\Setbar{B}{BW_L^*=0} & k=L\\ \Big(\sigma^{-1}(A)-\ones_Nb_k^*\Big)(W_k^*)^\dagger + \Setbar{B}{BW_k^*=0} & \textrm{else}\end{cases}
    \end{align*}
    which is a connected set in each case because of the following reasons: 
    1) the kernel of any matrix is connected, 
    2) the Minkowski-sum of two connected sets is connected by Proposition \ref{prop:minkowski}, 
    and 3) the image of a connected set under a continuous map is connected by Proposition \ref{prop:connected_continuous_map}.
    By repeating the similar argument for $k-1,\ldots,2$ 
    we conclude that $V\bydef G_2^{-1}\circ\ldots\circ G_k^{-1} (A)$ is connected.
    Lastly, we have
    \begin{align*}
	G_1^{-1}(V)=[X,\ones_N]^\dagger\sigma^{-1}(V) + \Setbar{B}{[X,\ones_N]B=0}
    \end{align*}
    which is also connected by the same arguments above.
    Thus $f^{-1}(A)$ is a connected set. 
    
    Overall, we have shown in this proof that the set of $(W_1,b_1)$ which realizes the same output at layer $k$ (given the parameters of other layers in between are fixed)
    is a connected set.
    Since both $(W_1^*,b_1^*)$ and $h\Big((W_l^*,b_l^*)_{l=2}^k, F_k(\theta)\Big)$ 
    belong to this solution set, there must exist a continuous path between them 
    on which the loss $\Phi$ is constant.
\end{proof}
The next lemma shows how to make the weight matrices full rank.
Its proof can be found in the appendix.
\begin{lemma}\label{lem:any_to_fullrankW}
    Let the conditions of Theorem \ref{thm:lin_data} hold.
    Let $\theta=(W_l,b_l)_{l=1}^L$ be any point in parameter space. 
    Then there is a continuous curve which starts from $\theta$ and ends at some $\theta'=(W_l',b_l')_{l=1}^L$ 
    so that $\theta'$ satisfies Property \ref{P1} and the loss $\Phi$ is constant on the curve.
\end{lemma}
\begin{proposition}\cite{Evard1994}\label{prop:connected_full_rank_matrices}
    The set of full rank matrices $A\in\RR^{m\times n}$ is connected for $m\neq n$.
\end{proposition}
\subsection{Proof of Theorem \ref{thm:lin_data}}
\begin{proof1}\textit{1.} 
    Let $L_\alpha$ be some sublevel set of $\Phi$.
    Let $\theta=(W_l,b_l)_{l=1}^L$ and $\theta'=(W_l',b_l')_{l=1}^L$ be arbitrary points in $L_\alpha$.
    Let $F_L=F_L(\theta)$ and $F_L'=F_L(\theta').$
    These two quantities are computed in the beginning and will never change during this proof.
    But when we write $F_L(\theta'')$ for some $\theta''$ we mean the network output evaluated at $\theta''.$
    The main idea is to construct two different continuous paths which simultaneously start from $\theta$ and $\theta'$
    and are entirely contained in $L_\alpha$ (this is done by making the loss on each individual path non-increasing),
    and then show that they meet at a common point in $L_\alpha$, which then implies that $L_\alpha$ is a connected set.
    
    First of all, we can assume that both $\theta$ and $\theta'$ satisfy Property \ref{P1}, because otherwise
    by Lemma \ref{lem:any_to_fullrankW} one can follow a continuous path from each point
    to arrive at some other point where this property holds and the loss on each path is invariant, 
    meaning that we still stay inside $L_\alpha.$
    As $\theta$ and $\theta'$ satisfy Property \ref{P1}, 
    all the weight matrices $(W_l)_{l=2}^L$ and $(W_l')_{l=2}^L$ have full rank,
    and thus by applying the second statement of Lemma \ref{lem:canonical_form} with $k=L$ and using the similar argument above, 
    we can simultaneously drive $\theta$ and $\theta'$ to the following points,
    \begin{align}\label{eq:thm:lin_data:theta'}
	&\!\!\theta=\Big(h\Big((W_l,b_l)_{l=2}^L, F_L\Big),(W_2,b_2),\ldots,(W_L,b_L)\Big),\nonumber\\
	&\!\!\theta'=\Big(h\Big((W_l',b_l')_{l=2}^L, F_L'\Big),(W_2',b_2'),\ldots,(W_L',b_L')\Big)
    \end{align}
    where $h:\Omega_2^*\times\ldots\times\Omega_L^*\times\RR^{N\times m}\to\Omega_1$ is the continuous map from Lemma \ref{lem:canonical_form}
    which satisfies 
    \begin{align}\label{eq:thm:lin_data:h}
	&\!F_L\Big(h\Big((\hat{W}_l,\hat{b}_l)_{l=2}^L,A\Big),(\hat{W}_l,\hat{b}_l)_{l=2}^L\Big)=A,\textrm{ for every}\\
	&\!\Big((\hat{W}_l,\hat{b}_l),\ldots,(\hat{W}_L,\hat{b}_L),A\Big)\in\Omega_2^*\times\ldots\times\Omega_L^*\times\RR^{N\times n_k}.\nonumber
    \end{align}
    Next, we construct a continuous path starting from $\theta$ on which the loss is constant
    and it holds at the end point of the path that all parameters from layer $2$ till layer $L$ are equal to the corresponding parameters of $\theta'.$
    Indeed, by applying Proposition \ref{prop:connected_full_rank_matrices} to the pairs of full rank matrices $(W_l,W_l')$ for every $l\in[2,L]$,
    we obtain continuous curves $W_2(\lambda),\ldots,W_L(\lambda)$
    so that $W_l(0)=W_l,W_l(1)=W_l'$ and $W_l(\lambda)$ has full rank for every $\lambda\in[0,1].$
    For every $l\in[2,L],$ let $c_l:[0,1]\to\Omega_l^*$ be the curve of layer $l$ defined as
    $$c_l\big(\lambda)=\Big(W_l(\lambda), (1-\lambda)b_l+\lambda b_l'\Big).$$
    We consider the curve $c:[0,1]\to\Omega$ given by
    \begin{align*}
	c(\lambda)=\Big( h\Big((c_l(\lambda))_{l=2}^L, F_L\Big), c_2(\lambda),\ldots,c_L(\lambda) \Big).
    \end{align*}
    Then one can easily check that $c(0)=\theta$ and $c$ is continuous as all the functions $h,c_2,\ldots,c_l$ are continuous. 
    Moreover, we have $\Big(c_2(\lambda),\ldots,c_L(\lambda)\Big)\in\Omega_2^*\times\ldots\times\Omega_L^*$
    and thus it follows from \eqref{eq:thm:lin_data:h} that $F_L(c(\lambda))=F_L$
    for every $\lambda\in[0,1],$ which leaves the loss invariant on $c.$
    
    Since the curve $c$ above starts at $\theta$ and has constant loss,
    we can reset $\theta$ to the end point of this curve, by setting $\theta=c(1)$,
    while keeping $\theta'$ from \eqref{eq:thm:lin_data:theta'}, which together give us
    \begin{align*}
	&\theta=\Big(h\Big((W_l',b_l')_{l=2}^L, F_L\Big),(W_2',b_2'),\ldots,(W_L',b_L')\Big),\\
	&\theta'=\Big(h\Big((W_l',b_l')_{l=2}^L, F_L'\Big),(W_2',b_2'),\ldots,(W_L',b_L')\Big).
    \end{align*}
    Now we note that the parameters of $\theta$ and $\theta'$ coincide at all layers except at the first layer.
    We will construct two continuous paths inside $L_\alpha$, say $c_1(\cdot)$ and $c_2(\cdot)$, 
    which starts from $\theta$ and $\theta'$ respectively ,
    and show that they meet at a common point in $L_\alpha.$
    Let $\hat{Y}\in\RR^{N\times m}$ be any matrix so that
    \begin{align}\label{eq:thm:lin_data:Yhat}
	\varphi(\hat{Y})\leq \min(\Phi(\theta),\Phi(\theta')).
    \end{align}
    Consider the curve $c_1:[0,1]\to\Omega$ defined as
    \begin{align*}
	c_1(\lambda)\!=\!\Big(h\Big((W_l',b_l')_{l=2}^L, (1-\lambda)F_L+\lambda\hat{Y}\Big),(W_l',b_l')_{l=2}^L\Big).
    \end{align*}
    Note that $c_1$ is continuous as $h$ is continuous, and it holds:
    \begin{align*}
	c_1(0)=\theta,\quad c_1(1)=\Big(h\Big((W_l',b_l')_{l=2}^L, \hat{Y}\Big),(W_l',b_l')_{l=2}^L\Big).
    \end{align*}
    It follows from the definition of $\Phi$, $c_1(\lambda)$ and \eqref{eq:thm:lin_data:h} that
    \begin{align*}
	\Phi(c_1(\lambda))
	=\varphi(F_L(c_1(\lambda)))
	=\varphi((1-\lambda)F_L+\lambda\hat{Y})
    \end{align*}
    and thus by convexity of $\varphi$,
    \begin{align*}
	\Phi(c_1(\lambda))
	&\leq (1-\lambda)\varphi(F_L) + \lambda\varphi(\hat{Y})\\
	&\leq(1-\lambda)\Phi(\theta) + \lambda\Phi(\theta)\nonumber
	=\Phi(\theta)\nonumber,
    \end{align*}
    which implies that $c_1[0,1]$ is entirely contained in $L_\alpha.$
    Similarly, we can also construct a curve $c_2(\cdot)$ inside $L_\alpha$ which starts at $\theta'$ and satisfies
    \begin{align*}
	c_2(0)=\theta',\; c_2(1)=\Big(h\Big((W_l',b_l')_{l=2}^L, \hat{Y}\Big),(W_l',b_l')_{l=2}^L\Big).
    \end{align*}
    So far, the curves $c_1$ and $c_2$ start at $\theta$ and $\theta'$ respectively and meet at the same point $c_1(1)=c_2(1).$
    
    Overall, we have shown that starting from any two points in $L_\alpha$
    we can find two continuous curves so that the loss is non-increasing on each curve,
    and these curves meet at a common point in $L_\alpha,$ and so $L_\alpha$ has to be connected.
    Moreover, the point where they meet satisfies $\Phi(c_1(1))=\varphi(\hat{Y}).$
    From \eqref{eq:thm:lin_data:Yhat}, $\varphi(\hat{Y})$ can be chosen arbitrarily small,
    and thus $\Phi$ can attain any value arbitrarily close to $p^*.$
    
    \textit{2.}  
    Let $C$ be a non-empty connected component of some level set, 
    i.e. $C\subseteq\Phi^{-1}(\alpha)$ for some $\alpha\in\RR.$
    Let $\theta=(W_l,b_l)_{l=1}^L\in C.$
    Similar as above,
    we first use Lemma \ref{lem:any_to_fullrankW} to find a continuous path from $\theta$ to some other point where $W_2$ attains full rank,
    and the loss is invariant on the path.
    From that point, we apply Lemma \ref{lem:canonical_form} with $k=2$ 
    to obtain another continuous path (with constant loss) which leads us to
    $\theta'\bydef\Big(h\Big((W_2,b_2), F_2(\theta)\Big), (W_2,b_2),\ldots,(W_L,b_L)\Big)$
    where $h:\Omega_2^*\to\Omega_1$ is a continuous map satisfying that
    \begin{align*}
	F_2\Big(h\Big((\hat{W}_2,\hat{b}_2),A\Big),(\hat{W}_l,\hat{b}_l)_{l=2}^L\Big)=A,
    \end{align*}
    for every point $(\hat{W}_l,\hat{b}_l)_{l=1}^L$ such that $\hat{W}_2$ has full rank,
    and every $A\in\RR^{N\times n_2}.$
    Note that $\theta'\in C$ as the loss is constant on the above paths.
    Consider the following continuous curve 
    $$c(\lambda)=\Big(h\Big((\lambda W_2,b_2),F_2(\theta)\Big),(\lambda W_2,b_2),\ldots,(W_L,b_L)\Big)$$
    for every $\lambda\geq 1.$
    This curve starts at $\theta'$ since $c(1)=\theta'.$ 
    Moreover $F_2(c(\lambda))=F_2(\theta)$ for every $\lambda\geq 1$
    and thus the loss is constant on this curve, meaning that the entire curve belongs to $C.$
    Lastly, since $W_2$ is full-rank, the curve $c$ is unbounded as $\lambda$ goes to infinity,
    thus $C$ is unbounded.
\end{proof1}
\section{Large Width of One of Hidden Layers Leads to No Bad Local Valleys}
In the previous section, we show that linearly independent training data essentially leads to connected sublevel sets.
In this section, we show the first application of this result in proving absence of bad local valleys on the loss landscape 
of deep and wide neural nets with arbitrary training data.
\begin{definition}\label{def:bad_valley}
    A local valley is a nonempty connected component of some strict sublevel set $L_{\alpha}^s\bydef\Setbar{\theta}{\Phi(\theta) < \alpha}.$
    A bad local valley is a local valley on which the training loss $\Phi$ cannot be made arbitrarily close to $p^*$.
\end{definition}
Intuitively, one can see that a small neighborhood of any suboptimal strict local minimum is a bad local valley.
However, we note that the notion of bad local valleys as analyzed in this paper has a more general meaning.
In particular, a bad local valley need not be restricted to any neighborhood of a bad local minimum
but can be any arbitrary subset in parameter space over which the infimum of the loss is strictly larger than the infimum over the entire space.
This is demonstrated via two simple examples as shown in Figure \ref{fig:bad_valleys}:
in the left example, the corresponding bad local valley contains multiple local minima, and in the right example, the bad valley does not even contain 
any local minimum.

\begin{figure}[ht]
\begin{center}
\includegraphics[height=0.18\columnwidth]{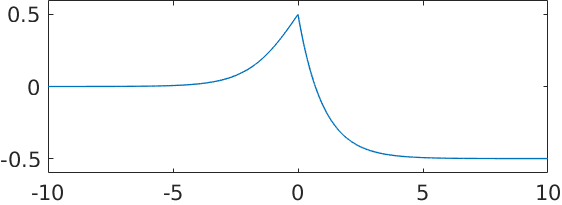}
\includegraphics[height=0.18\columnwidth]{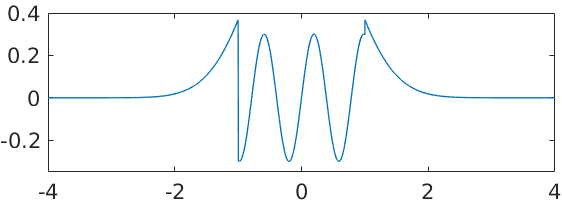}
\end{center}
\caption{
\textbf{Left}: an example function with exponential tails where local minima do NOT exist but local/global valleys still exist.
\textbf{Right}: a different function which satisfies every local minimum is a global minimum, 
but bad local valleys still exist at both infinities (exponential tails) where local search algorithms easily get stuck.
}
\label{fig:local_minima_bad_valley}
\end{figure}

Our main result in this section is stated as follows.
\begin{theorem}\label{thm:no_bad_valleys}
    Let Assumption \ref{ass:act1} and Assumption \ref{ass:act2} hold.
    Suppose that there exists a layer $k\in[1,L-1]$ such that $n_k\geq N$ and $n_{k+1}>\ldots> n_L.$
    Then the following hold:
    \begin{enumerate}
	\item The loss $\Phi$ has no bad local valleys.
	\item If $k\leq L-2$ then every local valley of $\Phi$ is unbounded.
    \end{enumerate}
\end{theorem}
The conditions of Theorem \ref{thm:no_bad_valleys} 
are satisfied for any strictly monotonic and piecewise linear activation function such as Leaky-ReLU (see Lemma \ref{lem:exp_act}).
We note that for Leaky-ReLU and other similar homogeneous activation functions, 
the second statement of Theorem \ref{thm:no_bad_valleys} is straightforward.
Indeed, if one scales all parameters of one hidden layer by some arbitrarily large factor $k>0$ 
and the weight matrix of the following layer by $1/k$ then the network output will be unchanged,
and so every connected component of every level set (and sublevel set) must extend to infinity and thus be unbounded.
However, for general non-homogeneous activation functions, this statement is non-trivial.

The first statement of Theorem \ref{thm:no_bad_valleys} implies that 
there is a continuous path from any point in parameter space 
on which the loss is non-increasing and gets arbitrarily close to $p^*.$
At this point, one might wonder that if a function satisfying ``every local minimum is a global minimum'' 
would automatically contain no bad local valleys.
Unfortunately this is not true in general.
Indeed, Figure \ref{fig:local_minima_bad_valley} shows two counter-examples 
where a function does not have any bad local mimina, but bad local valleys still exist.
The reason for this lies at the fact that bad local valleys 
in general need not contain any local minimum, or even critical point, although in theory they can have arbitrarily large volume.
Thus any pure results on global optimality of local minima with no further information on the loss 
would not be sufficient to guarantee convergence of local search algorithms to the bottom of the loss landscape,
especially if they are initialized in such non-optimal valleys.
Similar phenomenon has also been observed by \cite{Jascha2019}.

We note that while the statements of Theorem \ref{thm:no_bad_valleys} imply the absence of strict local minima and bounded local valleys,
they do not rule out the possibility of non-strict bad local minima.
Overall, the two statements of Theorem \ref{thm:no_bad_valleys} imply that every local valley must be an ``unbounded global valley''
in which the loss can attain any value arbitrarily close to $p^*.$

The high level proof idea for Theorem \ref{thm:no_bad_valleys} is that inside every local valley one can find a point
where the feature representations of all training samples are linearly independent at the wide hidden layer,
and thus an application of Theorem \ref{thm:lin_data} to the subnetwork from this wide layer till the output layer yields the result.
We list below several technical lemmas which are helpful to prove Theorem \ref{thm:no_bad_valleys}.
\begin{lemma}\label{lem:FW}
    Let $(F,W,\mathcal{I})$ be such that $F\in\RR^{N\times n}, W\in\RR^{n\times p},\rank(F)<n$
    and $\mathcal{I}\subset\Set{1,\ldots,n}$ be a subset of columns of $F$
    so that $\rank(F(:,\mathcal{I}))=\rank(F)$ and $\bar{\mathcal{I}}$ the remaining columns.
    Then there exists a continuous curve $c:[0,1]\to\RR^{n\times p}$ 
    which satisfies the following:
    \begin{enumerate}
	\item $c(0)=W$ and $Fc(\lambda)=FW,\,\forall\,\lambda\in[0,1].$
	\item The product $Fc(1)$ is independent of $F(:,\bar{\mathcal{I}}).$
    \end{enumerate}
\end{lemma}
\begin{lemma}\label{lem:span}
    Given $v\in\RR^n$ with $v_i\neq v_j\,\forall\,i\neq j$, 
    and $\sigma:\RR\to\RR$ satisfies Assumption \ref{ass:act2}. 
    Let $S\subseteq\RR^n$ be defined as $S=\Setbar{\sigma(v+b\ones_n)}{b\in\RR}.$
    Then it holds $\Span(S)=\RR^n.$
\end{lemma}
We recall the following standard result from topology (e.g., see \citet{Apostol1974}, Theorem 4.23, p. 82).
\begin{proposition}\label{prop:inverse_open}
    Let $f:\RR^m\to\RR^n$ be a continuous function.
    If $U\subseteq\RR^n$ is an open set then $f^{-1}(U)$ is also open.
\end{proposition}

\subsection{Proof of Theorem \ref{thm:no_bad_valleys}}
\begin{proof1}\textit{1.}
    Let $C$ be a connected component of some strict sublevel set $L_\alpha^s=\Phi^{-1}((-\infty,\alpha)),$
    for some $\alpha>p^*.$
    By Proposition \ref{prop:inverse_open}, $L_\alpha^s$ is an open set and thus $C$ must be open.
    
    \underline{Step 1: Finding a point inside $C$ where $F_k$ has full rank.}
    Let $\theta\in C$ be such that the pre-activation outputs at the first hidden layer are distinct for all training samples.
    Note that such $\theta$ always exist since Assumption \ref{ass:data} implies that
    the set of $W_1$ where this does not hold has Lebesgue measure zero, whereas $C$ has positive measure. 
    This combined with Assumption \ref{ass:act1} implies that the (post-activation) outputs at the first hidden layer
    are distinct for all training samples.
    Now one can view these outputs at the first layer as inputs to the next layer and argue similarly.
    By repeating this argument and using the fact that $C$ has positive measure, 
    we conclude that there exists $\theta\in C$ such that the outputs at layer $k-1$ are distinct for all training samples,
    \ie $(F_{k-1})_{i:}\neq (F_{k-1})_{j:}$ for every $i\neq j.$
    Let $V$ be the pre-activation output (without bias term) at layer $k,$ 
    in particular $V=F_{k-1}W_k=[v_1,\ldots,v_{n_k}]\in\RR^{N\times n_k}.$ 
    Since $F_{k-1}$ has distinct rows, one can easily perturb $W_k$ so that every column of $V$ has distinct entries.
    Note here that the set of $W_k$ where this does not hold has measure zero whereas $C$ has positive measure.
    Equivalently, $C$ must contain a point where every $v_j$ has distinct entries.
    To simplify notation, let $a=b_k\in\RR^{n_k},$ then by definition,
    \begin{align}\label{eq:thm:no_bad_valleys:F_k}
	F_k=[\sigma(v_1+\ones_N a_1),\ldots,\sigma(v_{n_k}+\ones_N a_{n_k})].
    \end{align}
    Suppose that $F_k$ has low rank, otherwise we are done.
    Let $r=\rank(F_k)<N\leq n_k$ and $\mathcal{I}\subset\Set{1,\ldots,n_k},|\mathcal{I}|=r$ be 
    the subset of columns of $F_k$ so that $\rank(F_k(:,\mathcal{I}))=\rank(F_k),$
    and $\bar{\mathcal{I}}$ the remaining columns.    
    By applying Lemma \ref{lem:FW} to $(F_k, W_{k+1},\mathcal{I}),$
    we can follow a continuous path with invariant loss (\ie entirely contained inside $C$) to arrive at some point 
    where $F_kW_{k+1}$ is independent of $F_k(:.\bar{\mathcal{I}}).$
    It remains to show how to change $F_k(:,\bar{\mathcal{I}})$ by modifying certain parameters so that $F_k$ has full rank.
    Let $p=|\bar{\mathcal{I}}|=n_k-r$ and $\bar{\mathcal{I}}=\Set{j_1,\ldots,j_p}.$ 
    From \eqref{eq:thm:no_bad_valleys:F_k} we have
    \begin{align*}
	F_k(:,\bar{\mathcal{I}})=[\sigma(v_{j_1}+\ones_N a_{j_1}),\ldots,\sigma(v_{j_p}+\ones_N a_{j_p})].
    \end{align*}
    Let $\mathrm{col}(\cdot)$ denotes the column space of a matrix.
    Then $\dim(\mathrm{col}(F_k(:,\mathcal{I})))=r<N.$
    Since $v_{j_1}$ has distinct entries, Lemma \ref{lem:span} implies that 
    there must exist $a_{j_1}\in\RR$ so that $\sigma(v_{j_1}+\ones_N a_{j_1})\notin\mathrm{col}(F_k(:,\mathcal{I})),$
    because otherwise $\Span\Setbar{\sigma(v_{j_1}+\ones_N a_{j_1})}{a_{j_1}\in\RR}\in\mathrm{col}(F_k(:,\mathcal{I}))$ 
    whose dimension is strictly smaller than $N$ and thus contradicts Lemma \ref{lem:span}. 
    So we pick one such value for $a_{j_1}$ and follow a direct line segment between its current value and the new value.
    Note that the loss is invariant on this segment since any changes on $a_{j_1}$ only affects $F_k(:,\bar{\mathcal{I}})$
    which however has no influence on the loss by above construction.
    Moreover, it holds at the new value of $a_{j_1}$ that $\rank(F_k)$ increases by $1.$ 
    Since $n_k\geq N$ by our assumption, it follows that $p\geq N-r$  
    and thus one can choose $\Set{a_{j_2},\ldots,a_{j_{N-r}}}$ in a similar way and finally obtain $\rank(F_k)=N.$
    
    \underline{Step 2: Applying Theorem \ref{thm:lin_data} to the subnetwork above $k$.}
    Suppose that we have found from previous step a $\theta=((W_l^*,b_l^*)_{l=1}^L)\in C$ so that $F_k$ has full rank.
    Let $g:\Omega_{k+1}\times\ldots\times\Omega_L$ be given as
    \begin{align}\label{eq:thm:no_bad_valleys:g}
	\!\!\!g\Big((W_l,b_l)_{l=k+1}^L\Big)\!=\!\Phi\Big((W_l^*,b_l^*)_{l=1}^k,(W_l,b_l)_{l=k+1}^L\Big)
    \end{align}
    We recall that $C$ is a connected component of $L_\alpha^s.$
    It holds $g\Big((W_l^*,b_l^*)_{l=k+1}^L\Big)=\Phi(\theta)\leq\alpha.$
    Now one can view $g$ as the new loss for the subnetwork from layer $k$ till layer $L$
    and $F_k$ can be seen as the new training data.
    Since $\rank(F_k)=N$ and $n_{k+1}>\ldots>n_L,$ 
    Theorem \ref{thm:lin_data} implies that $g$ has connected sublevel sets and 
    $g$ can attain any value arbitrarily close to $p^*.$
    Let $\epsilon\in(p^*,\alpha)$ and $(W_l',b_l')_{l=k+1}^L$ be any point such that $g\Big((W_l',b_l')_{l=k+1}^L\Big)\leq\epsilon.$
    Since both $(W_l^*,b_l^*)_{l=k+1}^L$ and $(W_l',b_l')_{l=k+1}^L$ belongs to the $\alpha$-sublevel set of $g,$
    which is a connected set, there must exist a continuous path from $(W_l^*,b_l^*)_{l=k+1}^L$ to $(W_l',b_l')_{l=k+1}^L$ 
    on which the value of $g$ is not larger than $\alpha.$
    This combined with \eqref{eq:thm:no_bad_valleys:g} implies that there is also a continuous path 
    from $\theta=\Big((W_l^*,b_l^*)_{l=1}^k,(W_l^*,b_l^*)_{l=k+1}^L\Big)$ 
    to $\theta'\bydef\Big((W_l^*,b_l^*)_{l=1}^k,(W_l',b_l')_{l=k+1}^L\Big)$
    on which the loss $\Phi$ is not larger than $\alpha.$
    Since $C$ is connected, it must hold $\theta'\in C.$
    Moreover, we have $\Phi(\theta')=g\Big((W_l',b_l')_{l=k+1}^L\Big)\leq\epsilon.$
    Since $\epsilon$ can be chosen arbitrarily small and close to $p^*$, we conclude that 
    the loss $\Phi$ can be made arbitrarily small inside $C,$ and thus $\Phi$ has no bad local valleys.
    
    \textit{2.}
    Let $C$ be a local valley, which by Definition \ref{def:bad_valley} is
    a connected component of some strict sublevel set $L_\alpha^s=\Phi^{-1}((-\infty,\alpha)).$
    According the the proof of the first statement above, 
    one can find a $\theta=(W_l^*,b_l^*)_{l=1}^L\in C$ 
    so that $F_k(\theta)$ has full rank.
    Now one can view $F_k(\theta)$ as the training data for the subnetwork from layer $k$ till layer $L.$
    The new loss is defined for this subnetwork as
    \begin{align*}
	g\Big((W_l,b_l)_{l=k+1}^L\Big) = \Phi\Big((W_l^*,b_l^*)_{l=1}^k,(W_l,b_l)_{l=k+1}^L\Big).
    \end{align*}
    By our assumptions, $\sigma$ satisfies Assumption \ref{ass:act1} and $n_{k+1}>\ldots>n_L,$
    thus the above subnetwork with the new loss $g$ and training data $F_k(\theta)$
    satisfy all the conditions of Theorem \ref{thm:lin_data},
    and so it follows that $g$ has unbounded level set components.
    Let $\beta\bydef g\Big((W_l^*,b_l^*)_{l=k+1}^L\Big)=\Phi(\theta)<\alpha.$
    Let $E$ be a connected component of the level set $g^{-1}(\beta)$
    which contains $(W_l^*,b_l^*)_{l=k+1}^L.$
    Let $D=\Setbar{\Big((W_l^*,b_l^*)_{l=1}^k,(W_l,b_l)_{l=k+1}^L\Big)}{(W_l,b_l)_{l=k+1}^L\in E}.$
    Then $D$ is connected and unbounded since $E$ is connected and unbounded.
    It holds for every $\theta'\in D$ that $\Phi(\theta')=\beta,$
    and thus $D\subseteq\Phi^{-1}(\beta)\subseteq L_\alpha^s,$
    where the last inclusion follows from $\beta<\alpha.$
    Moreover, we have $\theta=\Big((W_l^*,b_l^*)_{l=1}^k,(W_l^*,b_l^*)_{l=k+1}^L\Big)\in D$
    and also $\theta\in C$,
    it follows that $D\subseteq C$ since $C$ is already the maximal connected component of $L_\alpha^s.$
    Since $D$ is unbounded, $C$ must also be unbounded, which finishes the proof.
\end{proof1}

\begin{figure}
\begin{center}
\includegraphics[width=0.9\linewidth]{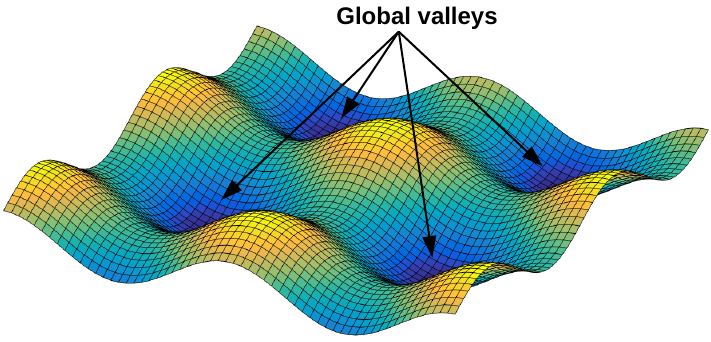}
\end{center}
\caption{A function which has no bad local valleys. Multiple valleys are caused due to the disconnectedness of sublevel sets.}
\label{fig:multiple_global_valleys}
\end{figure}
\section{Large Width of First Hidden Layer Leads to Connected Sublevel Sets}
In Theorem \ref{thm:no_bad_valleys},
we show that if one of the hidden layers has at least $N$ neurons then the loss function has no bad valleys -- see Figure \ref{fig:multiple_global_valleys} for an example.
Note that according to Theorem \ref{thm:no_bad_valleys} all the global valleys must be unbounded.
However for the purpose of illustration, we only plot bounded valleys in Figure \ref{fig:multiple_global_valleys} 
as the key point which we want to mention about this theorem here is that:
every local valley is a global valley, and there could exist multitude of them.
In this section, by analyzing a special case of the network where the first hidden layer has at least $2N$ neurons,
we can show further that there is only one such global valley.
\begin{theorem}\label{thm:connected_sublevel_sets}
    Let Assumption \ref{ass:act1} and Assumption \ref{ass:act2} hold.
    Suppose that $n_1\geq 2N$ and $n_2>\ldots> n_L.$
    Then every sublevel set of $\Phi$ is connected.
    Moreover, every connected component of every level set of $\Phi$ is unbounded.
\end{theorem}
Theorem \ref{thm:connected_sublevel_sets} shows a stronger property than Theorem \ref{thm:no_bad_valleys} 
as it not only implies that the loss has no bad local valleys but also that there is a unique valley.
As a result, all finite global minima of the loss (whenever they exist, e.g. for square loss) are automatically connected.
This can be seen as the generalization of \cite{Venturi2018} from one hidden layer networks and square loss to 
arbitrary deep nets and general convex loss functions.
Interestingly, some other recent work \cite{DraxlerEtal2018,GaripovEtal2018} have empirically shown that
different global minima of several existing CNN architectures can be connected by a continuous path 
on which the loss function has similar values.
While our current results are not directly applicable to these models,
we consider this as a stepping stone for such extensions in future work.
Lastly, the unboundedness property of level sets of $\Phi$
as shown in the second statement of Theorem \ref{thm:connected_sublevel_sets}
implies that the unique global valley of $\Phi$ in this case is unbounded, and $\Phi$ has no strict local extrema.
The proof of Theorem \ref{thm:connected_sublevel_sets} rests upon the following lemmas.
\begin{lemma}\label{lem:full_rank_F}
    Let $(X,W,b,V)\in\RR^{N\times d}\times\RR^{d\times n}\times\RR^n\times\RR^{n\times p}.$
    Let $\sigma:\RR\to\RR$ satisfy Assumption \ref{ass:act2}.
    Suppose that $n\geq N$ and $X$ has distinct rows.
    Let $Z=\sigma(XW+\ones_Nb^T)\,V.$
    There is a continuous curve $c:[0,1]\to\RR^{d\times n}\times\RR^n\times\RR^{n\times p}$ 
    with $c(\lambda)=(W(\lambda),b(\lambda),V(\lambda))$ satisfying:
    \begin{enumerate}
	\item $c(0)=(W,b,V).$
	\item $\sigma\Big(XW(\lambda))+\ones_Nb(\lambda)^T\Big)\,V(\lambda)=Z,\,\forall\,\lambda\in[0,1].$
	\item $\rank\Big(\sigma\Big(XW(1)+\ones_Nb(1)^T\Big)\Big)=N.$
    \end{enumerate}
\end{lemma}
\begin{lemma}\label{lem:equalization2N}
    Let $(X,W,V,W')\in\RR^{N\times d}\times\RR^{d\times n}\times\RR^{n\times p}\times\RR^{d\times n}.$
    Let $\sigma:\RR\to\RR$ satisfy Assumption \ref{ass:act2}.
    Suppose that $n\geq 2N$ and $\rank(\sigma(XW))=N, \rank(\sigma(XW'))=N.$
    Then there is a continuous curve $c:[0,1]\to\RR^{d\times n}\times\RR^{n\times p}$
    with $c(\lambda)=(W(\lambda),V(\lambda))$ which satisfies the following:
    \begin{enumerate}
	\item $c(0)=(W,V).$ 
	\item $\sigma(XW(\lambda))\,V(\lambda)=\sigma(XW)\,V,\quad\forall\,\lambda\in[0,1].$
	\item $W(1)=W'.$
    \end{enumerate}
\end{lemma}

\subsection{Proof of Theorem \ref{thm:connected_sublevel_sets}}
\begin{proof1}Let $\theta=(W_l,b_l)_{l=1}^L,\theta'=(W_l',b_l')_{l=1}^L$ be arbitrary points in some sublevel set $L_\alpha.$
    It is sufficient to show that there is a connected path between $\theta$ and $\theta'$
    on which the loss is not larger than $\alpha.$
    The output at the first layer is given by
    \begin{align*}
	F_1(\theta)&=\sigma([X,\ones_N][W_1^T,b_1]^T),\\
	F_1(\theta')&=\sigma([X,\ones_N][W_1'^T,b_1']^T).
    \end{align*}
    First, by applying Lemma \ref{lem:full_rank_F} to $(X,W_1,b_1,W_2),$ 
    we can assume that $F_1(\theta)$ has full rank, 
    because otherwise there is a continuous path starting from $\theta$ to some other point
    where the rank condition is fulfilled and the loss is invariant on the path,
    and so we can reset $\theta$ to this new point.
    Similarly, we can assume that $F_1(\theta')$ has full rank.
    
    Next, by applying Lemma \ref{lem:equalization2N} to the tuple $\Big([X,\ones_N],[W_1^T,b_1]^T,W_2,[W_1'^T,b_1']^T\Big),$
    and using the similar argument as above, we can drive $\theta$ to some other point where 
    the parameters of the first hidden layer agree with the corresponding values of $\theta'.$
    So we can assume w.l.o.g. that $(W_1,b_1)=(W_1',b_1').$
    Note that at this step we did not modify $\theta'$ but $\theta$ and thus $F_1(\theta')$ still has full rank.
    
    Once the first hidden layer of $\theta$ and $\theta'$ coincide, 
    one can view the output of this layer, say $F_1\bydef F_1(\theta)=F_1(\theta')$ with $\rank(F_1)=N$, 
    as the new training data for the subnetwork from layer $1$ till layer $L$ (given that $(W_1,b_1)$ is fixed).
    This subnetwork and the new data  $F_1$ satisfy all the conditions of Theorem \ref{thm:lin_data}, 
    and so it follows that the loss $\Phi$ restricted to this subnetwork has connected sublevel sets,
    which implies that there is a connected path between $(W_l,b_l)_{l=2}^L$ and $(W_l',b_l')_{l=2}^L$
    on which the loss is not larger than $\alpha.$
    This indicates that there is also a connected path between $\theta$ and $\theta'$  in $L_\alpha$ and so $L_\alpha$ must be connected.
    
    To show that every level set component of $\Phi$ is unbounded, let $\theta\in\Omega$ be an arbitrary point.
    Denote $F_1=F_1(\theta)$ and let $\mathcal{I}\subset\Set{1,\ldots,N}$ be such that $\rank(F_1(:,\mathcal{I}))=\rank(F_1).$
    Since $\rank(F_1)\leq\min(N,n_1)<n_1,$
    we can apply Lemma \ref{lem:FW} to the tuple $(F_1,W_2,\mathcal{I})$ to find a continuous path $W_2(\lambda)$ 
    which drives $\theta$ to some other point 
    where the output at 2nd layer $F_1 W_2$ is independent of $F_1(:,\bar{\mathcal{I}}).$
    Note that the network output at 2nd layer is invariant on this path and hence the entire path belongs to the same level set component with $\theta.$
    From that point, one can easily scale $(W_1(:,\bar{\mathcal{I}}),b_1(\bar{\mathcal{I}}))$ to arbitrarily large values without affecting the output.
    Since this path has constant loss and is unbounded, it follows that every level set component of $\Phi$ is unbounded.
\end{proof1}
\section{Extension to Other Activation Functions}
One way to extend our results from the previous sections to other activation functions such as ReLU and exponential linear unit
is to remove Assumption \ref{ass:act1} from the previous theorems, as shown next.
\begin{theorem}\label{thm:ReLU}
    Let $K=\min\Set{n_1,\ldots,n_{L-1}}.$
    Then all the following hold under Assumption \ref{ass:act2}:
    \begin{enumerate}
	\item If $K \geq N$ then the loss $\Phi$ has no bad local valleys.
	\item If $K \geq 2N$ then every sublevel set of $\Phi$ is connected.
    \end{enumerate}
\end{theorem}
It is clear that the conditions of Theorem \ref{thm:ReLU} are now much stronger than that of our previous theorems 
as all the hidden layers need to be wide enough.
Nevertheless, it is worth mentioning that this kind of technical conditions (i.e. all the hidden layers are sufficiently wide) 
have also been used in recent work \cite{AllenZhuEtal2018, DuEtAl2018_GD,ZouEtal2018} 
to establish convergence of gradient descent methods to a global minimum.
From a theoretical standpoint, these results seem to suggest that Leaky-ReLU might in general lead to a much ``easier'' loss surface than ReLU.

\begin{figure*}[ht]
\begin{center}
\includegraphics[width=0.4\linewidth]{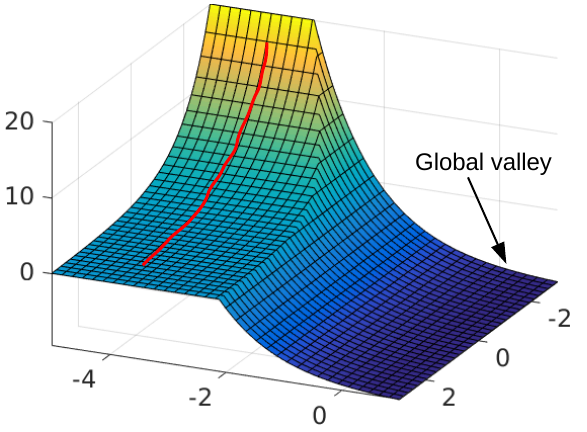}\hspace{20pt}
\includegraphics[width=0.4\linewidth]{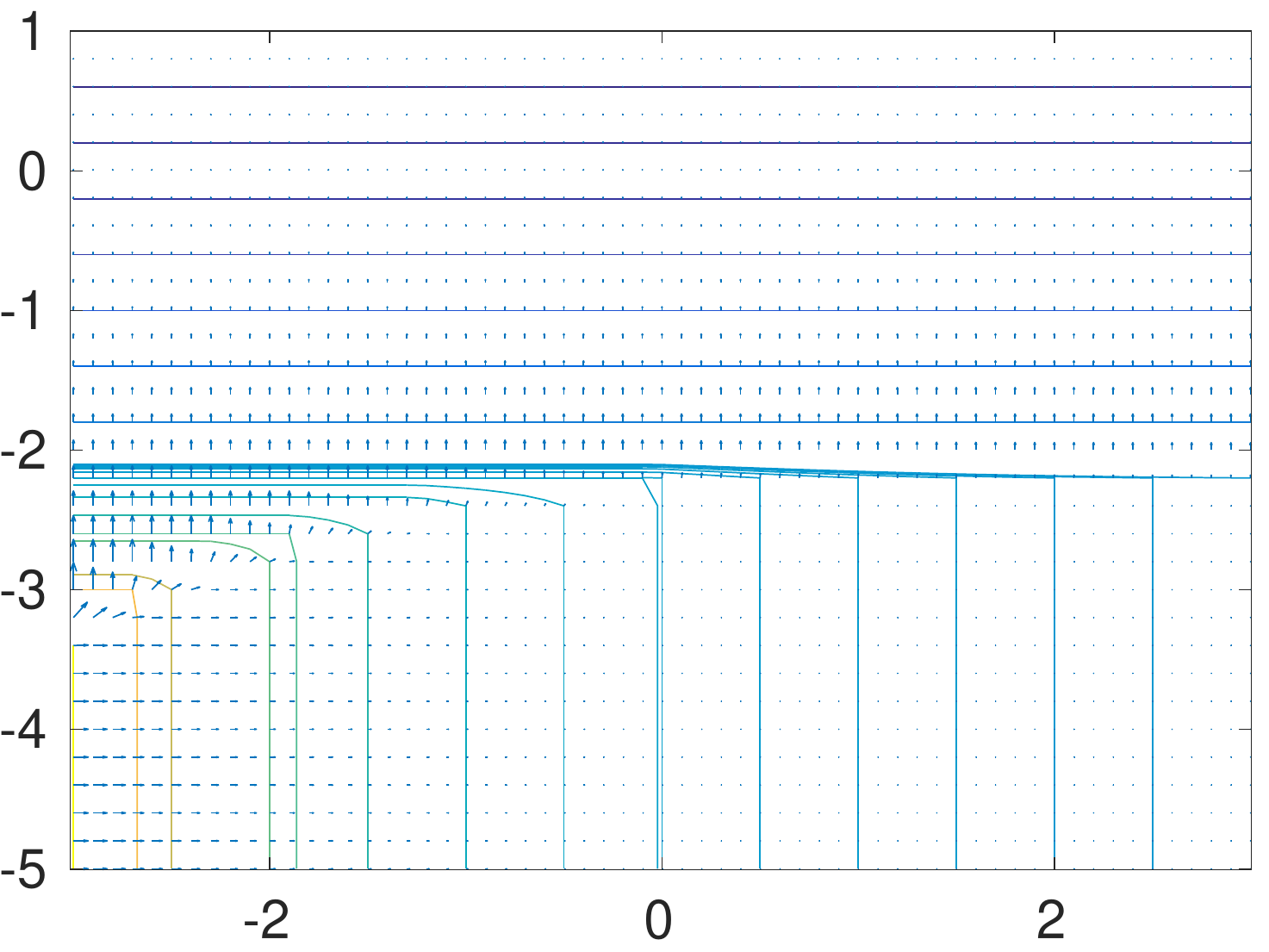}
\end{center}
\caption{Failure of gradient descent on a function with connected sublevel sets.
Left: A function with connected sublevel sets. 
Right: Contour and negative gradient map of the function on the left.
The plotted gradient flow (i.e. red curve) gets stuck infinitely as the function continues to decrease along that direction
and only reaches a suboptimal value at infinity.
Note that it also does not converge to any finite critical point 
as the function decreases exponentially in the direction of the trajectory.
}
\label{fig:divergence}
\end{figure*}

\section{Connectivity of Sublevel Sets Does Not Ensure Success of Gradient Descent}
In the previous sections, we have shown that over-parameterization in deep neural networks essentially leads to connected sublevel sets
of the loss function, which makes it more favorable to local search algorithms such as gradient descent.
However, as mentioned in the introduction, this result in general does not guarantee convergence of gradient descent to a global minimum.
First reason is due to the potential presence of saddle points.
Second, if gradient descent has bad initialization, 
then in the worst case the algorithm can get stuck infinitely as shown in Figure \ref{fig:divergence}.

\section{Related Work}
  Many interesting theoretical results have been developed on the loss surface of neural networks
  \cite{LivniEtal2014,Choro15,HaeVid2015,Hardt2017,XieEtal2017,Yun2017,LuKawaguchi2017,Pennington2017,ZhouLiang2018,LiangICML2018,LiaEtAl2018,ZhangShaoSala2018,Nouiehed2018,LaurentBrecht2018}.
  There is also a whole line of researches studying convergence of learning algorithms in training neural networks
  \citep{Andoni2014,Sedghi2015,Janzamin15,GauNgoHei2016,Brutzkus2017,Soltanolkotabi2017,Soudry17,Tian2018,Wang2018,JiTelgarsky2019,AroraEtal2019,ZhuLiang2018,Bartlett2018,ChizatBach2018}
  and others studying generalization properties, which is however beyond the scope of this paper.

  
  The closest existing result is the work by \cite{Venturi2018} 
  who 
  show that if the number of hidden neurons is greater than the intrinsic dimension of the network,
  defined as the dimension of some function space,
  then the loss has no spurious valley, and furthermore, if the number of hidden neurons is greater than two times the intrinsic dimension then 
  every sublevel set is connected.
  The results apply to one hidden layer networks with population risk and square loss.
  As admitted by the authors in the paper, an extension of such result, in particular the notion of intrinsic dimension, to multiple layer networks
  would require the number of neurons to grow exponentially with depth.
  Prior to this, \cite{SafSha2016} showed that for a class of deep fully connected networks with common loss functions,
  there exists a continuous descent path between specific pairs of points in parameter space which satisfy certain conditions,
  and these conditions can be shown to hold with probability $1/2$ as the width of the last hidden layer goes to infinity.
  
  
  Most closely related in terms of the setting are the work by \citep{Quynh2017,Quynh2018b} who analyze
  the optimization landscape of standard deep and wide (convolutional) neural networks for multiclass problem.
  They both assume that the network has a wide hidden layer $k$ with $n_k\geq N.$
  This condition has been recently relaxed to $n_1+\ldots+n_{L-1}\geq N$ by using flexible skip-connections \cite{Quynh2018c}.
  All of these results so far require real analytic activation functions, and thus are not applicable to the class of 
  piecewise linear activations analyzed in this paper.
  Moreover, while the previous work focus on global optimality of critical points,
  this paper characterizes sublevel sets of the loss function which gives us further insights and intuition on
  the underlying geometric structure of the optimization landscape.
%
\paragraph{Conclusion.}
We have shown that every sublevel set of the training loss function of a certain class of 
deep over-parameterized neural nets is connected and unbounded.
Our results hold for standard deep fully connected networks with piecewise linear activation functions, 
and general convex losses, e.g. square loss, cross-entropy loss and multiclass Hinge-loss.
We note that the property of connected sublevel sets as shown in this paper 
is satisfied by quasi-convex functions, and intuitively, 
this kind of functions are more favorable to (stochastic) gradient descent algorithms
than other wildly non-convex functions.
On the other hand, we also show that initialization conditions are very important to ensure
convergence of gradient descent to the bottom of the loss landscape.
In particular, even when all sublevel sets are connected, the algorithm might still get stuck infinitely or converge to a saddle point
depending on where it starts.

For future work, we find it interesting to study the influence of depth 
on optimization landscape of nonlinear networks.
We know previously that depth can improve the expressive power of nonlinear neural nets
\cite{Telgarsky2016,Eldan2016}, or accelerate the optimization of linear networks \cite{AroraEtal2018}.
However, whether or not depth has any influence on the geometry of sublevel sets of the loss function
still seems to be an entirely open problem.
\paragraph{Open problem.}
\textit{Given a reasonable fixed budget on each hidden layer width (i.e. significantly less than $N$),
is it possible to make the loss function of nonlinear neural networks 
be free of bad local valleys, or to make its sublevel sets become connected, by just increasing the depth?}


\section*{Acknowledgement} 
We would like to thank the anonymous reviewers for their support and helpful comments on the paper.

\bibliography{regul}
\bibliographystyle{icml2019}

\clearpage
\appendix

\section{Proof of Lemma \ref{lem:exp_act}}
    A function $\sigma:\RR\to\RR$ is continuous piecewise linear with at least two pieces if it can be represented as
    \begin{align*}
	\sigma(x) = a_i x + b_i,\quad \forall\,x\in(x_{i-1}, x_i),\,\forall\,i\in[1,n+1].
    \end{align*}
    for some $n\geq 1, x_0=-\infty<x_1<\ldots<x_n<x_{n+1}=\infty$ and $(a_i,b_i)_{i=1}^{n+1}.$
    We can assume that all the linear pieces agree at their intersection
    and there are no consecutive pieces with the same slope: $a_i\neq a_{i+1}$ for every $i\in[1,n].$
    Suppose by contradiction that $\sigma$ does not satisfy Assumption \ref{ass:act2},
    then there are non-zero coefficients $(\lambda_i,y_i)_{i=1}^m$ with $y_i\neq y_j (i\neq j)$ such that 
    $\sigma(x)=\sum_{i=1}^m\lambda_i\sigma(x-y_i)$ for every $x\in\RR.$
    We assume w.l.o.g. that $y_1<\ldots<y_m.$
    
    \underline{Case 1: $y_1>0.$}
    For every $x\in(-\infty, x_1)$ we have $\sigma(x)=a_1x+b_1=\sum_{i=1}^m\lambda_i(a_1(x-y_i)+b_1)$
    and thus by comparing the coefficients on both sides we obtain 
	$\sum_{i=1}^m\lambda_i a_1=a_1.$
    Moreover, for every $x\in\big(x_1,\min(x_1+y_1,x_2)\big)$ it holds $\sigma(x)=a_2x+b_2=\sum_{i=1}^m\lambda_i(a_1(x-y_i)+b_1)$
    and so
	$\sum_{i=1}^m\lambda_i a_1=a_2.$
    Thus $a_1=a_2,$ which is a contradiction.
    
    \underline{Case 2: $y_1<0.$}
    By definition, for $x\in(-\infty, x_1+y_1)$ we have $\sigma(x)=a_1x+b_1=\sum_{i=1}^m\lambda_i(a_1(x-y_i)+b_1)$
    and thus by comparing the coefficients on both sides we obtain 
    \begin{align}\label{eq:rd}
	\sum_{i=1}^m\lambda_i a_1=a_1.
    \end{align}
    For $x\in\big(x_1+y_1,\min(x_1+y_2,x_1,x_2+y_1)\big)$ 
    it holds 
    \begin{align*}
	\sigma(x)
	&=a_1x+b_1\\
	&=\lambda_1(a_2(x-y_1)+b_2)+\sum_{i=2}^m\lambda_i(a_1(x-y_i)+b_1)
    \end{align*}
    and thus by comparing the coefficients we have
    \begin{align*}
	\lambda_1 a_2+\sum_{i=2}^m\lambda_i a_1=a_1.
    \end{align*}
    This combined with \eqref{eq:rd} leads to $\lambda_1 a_1=\lambda_1 a_2,$
    and thus $a_1=a_2$ (since $\lambda_1\neq 0$) which is a contradiction.
    
    One can prove similarly for ELU \cite{ELU2016}
    $$\sigma(x)=\begin{cases}x&x\ge0\\ \alpha(e^x-1)&x<0\end{cases} \textrm{ where }\alpha>0.$$
    Suppose by contradiction that there exist non-zero coefficients $(\lambda_i,y_i)_{i=1}^m$ with $y_i\neq y_j(i\neq j)$
    such that $\sigma(x)=\sum_{i=1}^m\lambda_i\sigma(x-y_i),$ and assume w.l.o.g. that $y_1<\ldots<y_m.$
    If $y_m>0$ then for every $x\in(\max(0,y_{m-1}),y_m)$ it holds
    \begin{align*}
	&\sigma(x)=x=\lambda_m\alpha(e^{x-y_m}-1) + \sum_{i=1}^{m-1}\lambda_i(x-y_i)\\
	\Rightarrow\quad& e^x = \frac{xe^{y_m}-\sum_{i=1}^{m-1}\lambda_i(x-y_i)e^{y_m}}{\lambda_m\alpha} + e^{y_m}
    \end{align*}
    which is a contradiction since $e^x$ cannot be identical to any affine function on any open interval.
    Thus it must hold that $y_m<0.$ 
    But then for every $x\in(y_m,0)$ we have
    \begin{align*}
	&\sigma(x)=\alpha(e^x-1)=\sum_{i=1}^m\lambda_i(x-y_i)\\
	\Rightarrow\quad& e^x=\frac{1}{\alpha}\sum_{i=1}^m\lambda_i(x-y_i)+1 
    \end{align*}
    which is a contradiction for the same reason above.

\section{Proof of Proposition \ref{prop:connected_continuous_map}}
    Pick some $a,b\in f(A)$ and let $x,y\in A$ be such that $f(x)=a$ and $f(y)=b.$
    Since $A$ is connected, there is a continuous curve $r: [0,1]\to A$ so that $r(0)=x,r(1)=y.$  
    Consider the curve $f\circ r:[0,1]\to f(A),$ then it holds that $f(r(0))=a, f(r(1))=b.$  
    Moreover, $f\circ r$ is continuous as both $f$ and $r$ are continuous.
    Thus it follows from Definition \ref{def:connected_set} that $f(A)$ is a connected.

\section{Proof of Proposition \ref{prop:minkowski}}
    Let $x,y\in U+V$ then there exist $a,b\in U$ and $c,d\in V$
    such that $x=a+c,y=b+d.$
    Since $U$ and $V$ are connected sets, there exist two continuous curves $p:[0,1]\to U$
    and $q:[0,1]\to V$ such that 
    $p(0)=a,p(1)=b$ and $q(0)=c,q(1)=d.$
    Consider the continuous curve $r(t)\bydef p(t)+q(t)$ then we have $r(0)=a+c=x, r(1)=b+d=y$ and $r(t)\in U+V$ for every $t\in[0,1].$
    This implies that every two elements in $U+V$ can be connected by a continuous curve
    and thus $U+V$ must be a connected set.

\section{Proof of Lemma \ref{lem:any_to_fullrankW}}
    The idea is to make one weight matrix full rank at a time while keeping the others fixed (except the first layer).
    Each step is done by following a continuous path which leads to a new point where the rank condition is fulfilled
    while keeping the loss constant along the path.
    Each time when we follow a continuous path, we reset our starting point to the end point of the path and proceed.
    This is repeated until all the matrices $(W_l)_{l=2}^L$ have full rank.
    
    \underline{Step 1: Make $W_2$ full rank.}
    If $W_2$ has full rank then we proceed to $W_3.$
    Otherwise, let $\rank(W_2)=r<n_2<n_1.$
    Let $\mathcal{I}\subset\Set{1,\ldots,n_1}, |\mathcal{I}|=r$ 
    denote the set of indices of linearly independent rows of $W_2$ so that $\rank(W_2(\mathcal{I},:))=r.$ 
    Let $\bar{\mathcal{I}}$ denote the remaining rows of $W_2.$
    Let $E\in\RR^{(n_1-r)\times r}$ be a matrix such that $W_2(\bar{\mathcal{I}},:)=E W_2(\mathcal{I},:).$
    Let $P\in\RR^{n_1\times n_1}$ be a permutation matrix which permutes the rows of $W_2$ according to $\mathcal{I}$
    so that we can write
    \begin{align*}
	PW_2 = \begin{bmatrix}W_2(\mathcal{I},:)\\W_2(\bar{\mathcal{I}},:)\end{bmatrix}.
    \end{align*}
    We recall that $F_1(\theta)$ is the output of the network at the first layer, evaluated at $\theta$.
    Below we drop $\theta$ and just write $F_1$ as it is clear from the context.
    By construction of $P$, we have 
    $$F_1P^T=[F_1(:,\mathcal{I}),F_1(:,\bar{\mathcal{I}})].$$    
    The first step is to turn $W_1$ into a canonical form.
    In particular, the set of all possible solutions of $W_1$ which realizes the same the output $F_1$ at the first hidden layer 
    is characterized by $X^\dagger\big(\sigma^{-1}(F_1)-\ones_N b_1^T\big) + \ker(X)$
    where we denote, by abuse of notation, $\ker(X)=\Setbar{A\in\RR^{d\times n_1}}{XA=0}.$
    This solution set is connected because $\ker(X)$ is a connected set 
    and the Minkowski-sum of two connected sets is known to be connected,
    and so there exists a continuous path between every two solutions in this set on which the output $F_1$ is invariant.
    Obviously the current $W_1$ and $X^\dagger(\sigma^{-1}(F_1)-\ones_N b_1^T)$ are elements of this set,
    thus they must be connected by a continuous path on which the loss is invariant.
    So we can assume now that $W_1=X^\dagger(\sigma^{-1}(F_1)-\ones_N b_1^T).$ 
    Next, consider the curve:
    \begin{align*}	
	&W_1(\lambda)=X^\dagger\Big(\sigma^{-1}(A(\lambda))-\ones_Nb_1^T\Big), \textrm{where} \\
	&A(\lambda)=[F_1(:,\mathcal{I})+\lambda F_1(:,\bar{\mathcal{I}})E,(1-\lambda)F_1(:,\bar{\mathcal{I}})]\,P.
    \end{align*}
    This curves starts at $\theta$ since $W_1(0)=W_1$, 
    and it is continuous as $\sigma$ has a continuous inverse by Assumption \ref{ass:act1}.
    Using $XX^\dagger=\Id$, one can compute the pre-activation output (without bias term) at the second layer as
    \begin{align*}
	\sigma\big(XW_1(\lambda)+\ones_Nb_1^T\big)\, W_2 = A(\lambda)\, W_2 = F_1 W_2,
    \end{align*}
    which implies that the loss is invariant on this curve, 
    and so we can take its end point $W_1(1)$ as a new starting point:
    \begin{align*}
	&W_1=X^\dagger\Big(\sigma^{-1}(A)-\ones_Nb_1^T\Big), \textrm{where} \\
	&A=[F_1(:,\mathcal{I})+F_1(:,\bar{\mathcal{I}})E,\,\textbf{0}]\,P.
    \end{align*}
    Now, the output at second layer above, given by $AW_2$, is independent of $W_2(\bar{\mathcal{I}},:)$ because it is canceled
    by the zero component in $A.$ 
    Thus one can easily change $W_2(\bar{\mathcal{I}},:)$ so that $W_2$ has full rank while still keeping the loss invariant.
    
    \underline{Step 2: Using induction to make $W_3,\ldots,W_L$ full rank.}
    Let $\theta=(W_l,b_l)_{l=2}^L$ be our current point.
    Suppose that all the matrices $(W_l)_{l=2}^k$ already have full rank for some $k\geq 2$
    then we show below how to make $W_{k+1}$ full rank.
    We write $F_k$ to denote $F_k(\theta).$
    By the second statement of Lemma \ref{lem:canonical_form}, 
    we can follow a continuous path (with invariant loss) to drive $\theta$ to the following point:
    \begin{align}\label{eq:lem:any_to_fullrankW:theta}
	\theta\bydef\Big(h\Big((W_l,b_l)_{l=2}^k, F_k\Big),(W_l,b_l)_{l=2}^L\Big)
    \end{align}
    where $h:\Omega_2^*\times\ldots\times\Omega_k^*\times\RR^{N\times n_k}$ is the continuous map from Lemma \ref{lem:canonical_form} 
    which satisfies for every $A\in\RR^{N\times n_k},$
    \begin{align}\label{eq:lem:any_to_fullrankW:prop_h}
	F_k\Big(h\big((W_l,b_l)_{l=2}^k,A\big),(W_l,b_l)_{l=2}^k\Big)=A.
    \end{align}

    Now, if $W_{k+1}$ already has full rank then we are done, otherwise we follow the similar steps as before.
    Indeed, let $r=\rank(W_{k+1})<n_{k+1}<n_k$ and $\mathcal{I}\subset\Set{1,\ldots,n_k},|\mathcal{I}|=r$ the set of indicies of 
    $r$ linearly independent rows of $W_{k+1}$.
    Then there is a permutation matrix $P\in\RR^{n_k\times n_k}$ and some matrix $E\in\RR^{(n_k-r)\times r}$ so that
    \begin{align}\label{eq:lem:any_to_fullrankW:PW_k+1}
	\!\!\!\!PW_{k+1}\!=\!\begin{bmatrix}W_{k+1}(\mathcal{I},:)\\W_{k+1}(\bar{\mathcal{I}},:)\end{bmatrix},
	W_{k+1}(\bar{\mathcal{I}},:)\!=\!EW_{k+1}(\mathcal{I},:).
    \end{align}
    Moreover it holds 
    \begin{align}\label{eq:lem:any_to_fullrankW:FkPT}
	F_kP^T=[F_k(:,\mathcal{I}),F_k(:,\bar{\mathcal{I}})].
    \end{align}
    Consider the following curve $c:[0,1]\to\Omega$ which continuously update $(W_1,b_1)$ while keeping other layers fixed:
    \begin{align*}
	&c(\lambda) = \Big( h\Big((W_l,b_l)_{l=2}^k, A(\lambda)\Big),(W_2,b_2),\ldots,(W_L,b_L) \Big),\\
	&\textrm{where } A(\lambda)=[F_k(:,\mathcal{I})+\lambda F_k(:,\bar{\mathcal{I}})E,(1-\lambda)F_k(:,\bar{\mathcal{I}})]\,P.
    \end{align*}
    It is clear that $c$ is continuous as $h$ is continuous.
    One can easily verify that $c(0)=\theta$ by using \eqref{eq:lem:any_to_fullrankW:FkPT} and \eqref{eq:lem:any_to_fullrankW:theta}.
    The pre-activation output (without bias term) at layer $k+1$ for every point on this curve is given by
    \begin{align*}
	F_k(c(\lambda))\,W_{k+1} = A(\lambda) W_{k+1} = F_k W_{k+1},\,\;\forall\,\lambda\in[0,1],
    \end{align*}
    where the first equality follows from \eqref{eq:lem:any_to_fullrankW:prop_h} 
    and the second follows from \eqref{eq:lem:any_to_fullrankW:PW_k+1} and \eqref{eq:lem:any_to_fullrankW:FkPT}.
    As the loss is invariant on this curve, we can take its end point $c(1)$ as a new starting point:
    \begin{align*}
	&\theta \bydef \Big( h\Big((W_l,b_l)_{l=2}^k, A\Big),(W_2,b_2),\ldots,(W_L,b_L) \Big),\\
	&\textrm{where } A=[F_k(:,\mathcal{I})+F_k(:,\bar{\mathcal{I}})E,\,\textbf{0}]\,P.
    \end{align*}
    At this point, the output at layer $k+1$ as mentioned above is given by $AW_{k+1},$ which is independent of $W_{k+1}(\bar{\mathcal{I}},:)$
    since it is canceled out by the zero component in $A,$
    and thus one can easily change the submatrix $W_{k+1}(\bar{\mathcal{I}},:)$ so that $W_{k+1}$ has full rank while
    leaving the loss invariant.
    
    Overall, by induction we can make all the weight matrices $W_2,\ldots,W_L$ full rank 
    by following several continuous paths on which the loss is constant, which finishes the proof.

\section{Proof of Lemma \ref{lem:FW}}
    Let $r=\rank(F)<n.$ 
    Since $\mathcal{I}$ contains $r$ linearly independent columns of $F$,
    the remaining columns must lie on their span.
    In other words, there exists $E\in\RR^{r\times(n-r)}$ 
    so that $F(:,\bar{\mathcal{I}})=F(:,\mathcal{I})\,E.$
    Let $P\in\RR^{n\times n}$ be a permutation matrix which permutes the columns of $F$ according to $\mathcal{I}$ so that 
    we can write $F=[F(:,\mathcal{I}),F(:,\bar{\mathcal{I}})]\,P.$
    Consider the continuous curve $c:[0,1]\to\RR^{n\times p}$ defined as
    \begin{align*}
	c(\lambda)=P^T\,\begin{bmatrix}W(\mathcal{I},:)+\lambda E\,W(\bar{\mathcal{I}},:)\\ (1-\lambda)W(\bar{\mathcal{I}},:)\end{bmatrix}, \,\forall\,\lambda\in[0,1].
    \end{align*}
    It holds $c(0)=P^T\,\begin{bmatrix}W(\mathcal{I},:)\\W(\bar{\mathcal{I}},:)\end{bmatrix}=W.$
    For every $\lambda\in[0,1]:$
    \begin{align*}
	Fc(\lambda) 
	&= [F(:,\mathcal{I}),F(:,\bar{\mathcal{I}})]\,PP^T\, \begin{bmatrix}W(\mathcal{I},:)+\lambda E\,W(\bar{\mathcal{I}},:)\\ (1-\lambda)W(\bar{\mathcal{I}},:)\end{bmatrix}\\
	&= F(:,\mathcal{I}) W(\mathcal{I},:) + F(:,\bar{\mathcal{I}}) W(\bar{\mathcal{I}},:) = FW .
    \end{align*}
    Lastly, we have
    \begin{align*}
	Fc(1) 
	&= [F(:,\mathcal{I}),F(:,\bar{\mathcal{I}})]\,PP^T\,\begin{bmatrix}W(\mathcal{I},:)+ EW(\bar{\mathcal{I}},:)\\ \textbf{0}\end{bmatrix}\\
	&= F(:,\mathcal{I}) W(\mathcal{I},:) + F(:,\mathcal{I}) E\, W(\bar{\mathcal{I}},:)
    \end{align*}
    which is independent of $F(:,\bar{\mathcal{I}}).$

\section{Proof of Lemma \ref{lem:span}}
    Suppose by contradiction that $\dim(\Span(S))<n.$
    Then there exists $\lambda\in\RR^n,\lambda\neq 0$ such that $\lambda\perp\Span(S),$
    and thus it holds $\sum_{i=1}^n \lambda_i\sigma(v_i+b)=0$ for every $b\in\RR.$
    We assume w.l.o.g. that $\lambda_1\neq 0$ then it holds $$\sigma(v_1+b)=-\sum_{i=2}^n \frac{\lambda_i}{\lambda_1}\sigma(v_i+b),\quad\forall\,b\in\RR.$$
    By a change of variable, we have $$\sigma(c)=-\sum_{i=2}^n\frac{\lambda_i}{\lambda_1}\sigma(c+v_i-v_1),\quad\forall\,c\in\RR,$$
    which contradicts Assumption \ref{ass:act2}.
    Thus $\Span(S)=\RR^n.$

\section{Proof of Lemma \ref{lem:full_rank_F}}
    Let $F=\sigma(XW+\ones_Nb^T)\in\RR^{N\times n}.$
    If $F$ already has full rank then we are done.
    Otherwise let $r=\rank(F)<N\leq n.$
    Let $\mathcal{I}$ denote a set of column indices of $F$ so that $\rank(F(:,\mathcal{I}))=r$
    and $\bar{\mathcal{I}}$ the remaining columns.
    By applying Lemma \ref{lem:FW} to $(F,V,\mathcal{I}),$
    we can find a continuous path $V(\lambda)$ so that we will arrive at some point 
    where $FV(\lambda)$ is invariant on the path and it holds at the end point of the path that 
    $FV$ is independent of $F(:,\bar{\mathcal{I}}).$
    This means that we can arbitrarily change the values of $W(:,\bar{\mathcal{I}})$ and $b(\bar{\mathcal{I}})$ 
    without affecting the value of $Z$, 
    because any changes of these variables are absorbed into $F(:,\bar{\mathcal{I}})$ which anyway has no influence on $FV.$
    Thus it is sufficient to show that there exist $W(:,\bar{\mathcal{I}})$ and $b(\bar{\mathcal{I}})$ for which $F$ has full rank.
    Let $p=n-r$ and $\bar{\mathcal{I}}=\Set{j_1,\ldots,j_p}.$ 
    Let $A=XW$ then $A(:,\bar{\mathcal{I}})\bydef [a_{j_1},\ldots,a_{j_p}]=XW(:,\bar{\mathcal{I}}).$
    By assumption $X$ has distinct rows, one can choose $W(:,\bar{\mathcal{I}})$ so that 
    each $a_{j_k}\in\RR^N$ has distinct entries.
    Then we have
    \begin{align*}
	F(:,\bar{\mathcal{I}})=[\sigma(a_{j_1}+\ones_N b_{j_1}),\ldots,\sigma(a_{j_p}+\ones_N b_{j_p})].
    \end{align*}
    Let $\mathrm{col}(\cdot)$ denotes the column space of a matrix.
    It holds $\dim(\mathrm{col}(F(:,\mathcal{I})))=r<N.$
    Since $a_{j_1}$ has distinct entries, Lemma \ref{lem:span} implies that
    there must exist $b_{j_1}\in\RR$ so that $\sigma(a_{j_1}+\ones_N b_{j_1})\notin\mathrm{col}(F(:,\mathcal{I})),$
    because otherwise $\Span\Setbar{\sigma(a_{j_1}+\ones_N b_{j_1})}{b_{j_1}\in\RR}\in\mathrm{col}(F(:,\mathcal{I}))$ 
    whose dimension is strictly smaller than $N$, which contradicts Lemma \ref{lem:span}. 
    So it means that there is $b_{j_1}\in\RR$ so that $\rank(F)$ increases by $1.$
    By assumption $n\geq N,$ it follows that $p\geq N-r,$ 
    and thus we can choose $\Set{b_{j_2},\ldots,b_{j_{N-r}}}$ similarly to obtain $\rank(F)=N.$

\section{Proof of Lemma \ref{lem:equalization2N}}
    We need to show that there is a continuous path from $(W,V)$ to $(W',V')$ for some $V'\in\RR^{n\times p},$
    so that the output function, defined by $Z\bydef\sigma(XW)V,$ is invariant along the path.
    Let $F=\sigma(XW)\in\RR^{N\times n}$ and $F'=\sigma(XW').$
    It holds $Z=FV.$
    Let $I$ resp. $I'$ denote the maximum subset of linearly independent columns of $F$ resp. $F'$ 
    so that $\rank(F(:,I))=\rank(F(:,I'))=N,$
    and $\bar{I}$ and $\bar{I'}$ be their complements.
    By the rank condition, we have $|I|=|I'|=N.$
    Since $\rank(F)=N<n,$ we can apply Lemma \ref{lem:FW} to the tuple $(F,V,I)$
    to arrive at some point where the output $Z$ is independent of $F(:,\bar{I}).$
    From here, we can update $W(:,\bar{I})$ arbitrarily so that it does not affect $Z$
    because any change to these weights only lead to changes on $F(:,\bar{I})$ which however has no influence on $Z.$
    So by taking a direct line segment from the current value of $W(:,\bar{I})$ to $W'(:,I'),$ 
    we achieve $W(:,\bar{I})=W'(:,I').$
    We refer to this step below as a copy step.
    Note here that since $n\geq 2N$ by assumption, we must have $|\bar{I}|\geq|I'|.$
    Moreover, if $|\bar{I}|>|I'|$ then we can simply ignore the redundant space in $W(:,\bar{I}).$
    
    Now we already copy $W'(:,I')$ into $W(:,\bar{I}),$
    so it holds that $\rank(F(:,\bar{I}))=\rank(F'(:,I'))=N.$
    Let $K=I'\cap\bar{I}$ and $J=I'\cap I$ be disjoint subsets so that $I'=K\cup J.$ 
    Suppose w.l.o.g. that the above copy step has been done in such a way that $W(:,\bar{I}\cap I')=W'(:,K).$
    Now we apply Lemma \ref{lem:FW} to $(F,V,\bar{I})$ to arrive at some point
    where $Z$ is independent of $F(:,I),$
    and thus we can easily obtain $W(:,I\cap I')=W'(:,J)$ by taking a direct line segment between these weights.
    So far, all the rows of $W'(:,K\cup J)$ have been copied into $W(:,I')$ at the right positions
    so we obtain that $W(:,I')=W'(:,I').$
    It follows that $\rank(F(:,I'))=\rank(F'(:,I'))=N$ and 
    thus we can apply Lemma \ref{lem:FW} to $(F,V,I')$ to arrive at some other point
    where $Z$ is independent of $F(:,\bar{I'}).$
    From here we can easily obtain $W(:,\bar{I'})=W'(:,\bar{I'})$ by taking a direct line segment between these variables.
    Till now we already have $W=W'.$
    Moreover, all the paths which we have followed leave the output $Z$ invariant.

\section{Proof of Theorem \ref{thm:ReLU}}
    \underline{Case 1}: $\min \Set{n_1,\ldots,n_{L-1}} \geq N.$
    	Let $\theta=(W_l,b_l)_{l=1}^L$ be an arbitrary point of some strict sublevel set $L_\alpha^s,$ for some $\alpha>p^*.$
	We will show that there is a continuous descent path starting from $\theta$ 
	on which the loss is non-increasing and gets arbitrarily close to $p^*.$
	Indeed, for every $\epsilon$ arbitrarily close to $p^*$ and $\epsilon\leq\alpha,$
	let $\hat{Y}\in\RR^{N\times m}$ be such that $\varphi(\hat{Y})\leq\epsilon.$
	Since $X$ has distinct rows, $n_1\geq N,$
	and the activation $\sigma$ satisfies Assumption \ref{ass:act2},
	an application of Lemma \ref{lem:full_rank_F} to $(X,W_1,b_1,W_2)$ shows that
	there is a continuous path with constant loss which leads $\theta$ to some other point where the output at the first hidden layer is full rank.
	So we can assume w.l.o.g. that it holds for $\theta$ that $\rank(F_1)=N.$
	By assumption $n_1\geq N$ and $F_1\in\RR^{N\times n_1}$, it follows that $F_1$ must have distinct rows,
	and thus by applying Lemma \ref{lem:full_rank_F} again to $(F_1,W_2,b_2,W_3)$
	we can assume w.l.o.g. that $\rank(F_2)=N.$
	By repeating this argument to higher layers using our assumption on the width, 
	we can eventually arrive at some $\theta=(W_l,b_l)_{l=1}^L$ where $\rank(F_{L-1})=N.$
	Thus there must exist $W_{L-1}^*\in\RR^{n_{L-1}\times m}$ so that $F_{L-1}W_L^*=\hat{Y}-\ones_Nb_L^T.$ 
	Consider the line segment $W_L(\lambda)=(1-\lambda)W_L+\lambda W_L^*,$
	then it holds by convexity of $\varphi$ that
	\begin{align*}
	    &\Phi\Big((W_l,b_l)_{l=1}^{L-1}, (W_L(\lambda),b_L)\Big)\\
	    =&\varphi\Big(F_{L-1} W_L(\lambda)+\ones_Nb_L^T\Big)\\
	    =&\varphi\Big((1-\lambda)(F_{L-1}W_L+\ones_Nb_L^T) + \lambda(F_{L-1}W_L^*+\ones_Nb_L^T)\Big)\\
	    \leq& (1-\lambda)\varphi(F_L) + \lambda\varphi(\hat{Y})\\
	    <& (1-\lambda)\alpha + \lambda\epsilon
	    \leq \alpha.
	\end{align*}
	Thus the whole line segment is contained in $L_\alpha^s.$
	By plugging $\lambda=1$ we obtain $\Big((W_l,b_l)_{l=1}^{L-1},(W_L^*,b_L)\Big)\in L_\alpha^s.$
	Moreover, it holds $\Phi\Big((W_l,b_l)_{l=1}^{L-1},(W_L^*,b_L)\Big)=\varphi(\hat{Y})\leq\epsilon.$
	As $\epsilon$ can be chosen arbitrarily close to $p^*,$
	we conclude that $\Phi$ can be made arbitrarily close to $p^*$ in every strict sublevel set 
	which implies that $\Phi$ has no bad local valleys.
	
    \underline{Case 2}: $\min \Set{n_1,\ldots,n_{L-1}} \geq 2N.$
	Our first step is similar to the first step in the proof of Theorem \ref{thm:connected_sublevel_sets},
	which we repeat below for completeness.
	Let $\theta=(W_l,b_l)_{l=1}^L,\theta'=(W_l',b_l')_{l=1}^L$ be arbitrary points in some sublevel set $L_\alpha.$
	It is sufficient to show that there is a connected path between $\theta$ and $\theta'$
	on which the loss is not larger than $\alpha.$
	In the following, we denote $F_k$ and $F_k'$ as the output at a layer $k$ for $\theta$ and $\theta'$ respectively.
	The output at the first layer is:
	\begin{align*}
	    F_1=\sigma([X,\ones_N][W_1^T,b_1]^T),\\
	    F_1'=\sigma([X,\ones_N][W_1'^T,b_1']^T).
	\end{align*}
	By applying Lemma \ref{lem:full_rank_F} to $(X,W_1,b_1,W_2)$ and $(X,W_1',b_1',W_2')$ 
	we can assume w.l.o.g. that both $F_1$ and $F_1'$ have full rank, since otherwise
	there is a continuous path starting from each point 
	and leading to some other point where the rank condition is fulfilled 
	and the network output at second layer is invariant on the path.
	Once $F_1$ and $F_1'$ have full rank, 
	we can apply Lemma \ref{lem:equalization2N} to $\Big([X,\ones_N],[W_1^T,b_1]^T,W_2,[W_1'^T,b_1']^T\Big)$
	in order to drive $\theta$ to some other point where the parameters of the first layer are all equal to the corresponding ones of $\theta'.$
	So we can assume w.l.o.g. that $(W_1,b_1)=(W_1',b_1').$
	
	Once the network parameters of $\theta$ and $\theta'$ coincide at the first hidden layer,
	we can view the output of this layer, which is equal for both points (i.e., $F_1=F_1'$),
	as the new training data for the subnetwork from layer $2$ till layer $L.$
	Same as before, we first apply Lemma \ref{lem:full_rank_F} to $(F_1,W_2,b_2,W_3)$ and $(F_1',W_2',b_2',W_3')$ 
	to drive $\theta$ and $\theta'$ respectively to other new points where both $F_2$ and $F_2'$ have full rank.
	Note that this path only acts on $(W_2,b_2,W_3)$ and thus leaves everything else below layer $2$ invariant, in particular
	we still have $F_1=F_1'.$
	Then we can apply Lemma \ref{lem:equalization2N} again to the tuple $\Big([F_1,\ones_N],[W_2^T,b_2]^T,W_3,[W_2'^T,b_2']^T\Big)$
	to drive $\theta$ to some other point where $(W_2,b_2)=(W_2',b_2').$
	
	By repeating the above argument to the last hidden layer, 
	we can make all network parameters of $\theta$ and $\theta'$ coincide for all layers, except the output layer.
	In particular, the path that each $\theta$ and $\theta'$ has followed has invariant loss.
	The output of the last hidden layer for these points is $A\bydef F_{L-1}=F_{L-1}'.$
	The loss at these two points can be rewritten as 
	\begin{align*}
	    &\Phi(\theta)=\varphi\Big([A,\ones_N]\begin{bmatrix}W_L\\b_L^T\end{bmatrix}\Big),\\
	    &\Phi(\theta')=\varphi\Big([A,\ones_N]\begin{bmatrix}W_L'\\b_L'^T\end{bmatrix}\Big).
	\end{align*}
	By convexity of $\varphi$, it follows that the line segment 
	$$(1-\lambda)\begin{bmatrix}W_L\\b_L^T\end{bmatrix}+\lambda \begin{bmatrix}W_L'\\b_L'^T\end{bmatrix}$$
	must yield a continuous path between $(W_L,b_L)$ and $(W_L',b_L')$
	where the loss of every point along this path is upper bounded by $\max(\Phi(\theta),\Phi(\theta'))$, 
	which is clearly not larger than $\alpha.$
	Thus the entire line segment must belong to $L_\alpha$ which implies that $L_\alpha$ is connected.

\end{document}